\documentclass{article}
\usepackage{microtype}
\usepackage{graphicx}
\usepackage{subcaption}
\usepackage{booktabs} % for professional tables

\usepackage{hyperref}

\usepackage[preprint]{icml2026}
\usepackage[disable,textsize=tiny]{todonotes}

\usepackage{amsmath}
\usepackage{amssymb}
\usepackage{mathtools}
\usepackage{amsthm}

\usepackage[capitalize,noabbrev]{cleveref}

\theoremstyle{plain}
\newtheorem{theorem}{Theorem}[section]

\newtheorem{corollary}{Corollary}[section]
\theoremstyle{definition}
\newtheorem{definition}{Definition}[section]
\newtheorem{assumption}{Assumption}[section]
\theoremstyle{remark}

\icmltitlerunning{Multi-Head Attention is a Multi-Player Game}

\begin{document}

\twocolumn[
  \icmltitle{Multi-Head Attention is a Multi-Player Game}



  \icmlsetsymbol{equal}{*}

  \begin{icmlauthorlist}
    \icmlauthor{Kushal Chakrabarti}{ow,spc}
    \icmlauthor{Nirmal Balachundar}{spc}
  \end{icmlauthorlist}

  \icmlaffiliation{ow}{Obviously Wrong, LLC, San Francisco, CA, USA}
  \icmlaffiliation{spc}{South Park Commons, San Francisco, CA, USA}

  \icmlcorrespondingauthor{Kushal Chakrabarti}{kushalc@obviouslywrong.org}

  \icmlkeywords{Machine Learning, Transformers, Attention, Game Theory, Hallucination}

  \vskip 0.3in
]



\printAffiliationsAndNotice{}  % no special notice (required even if empty)

\begin{abstract}
  Modern transformer attention is internally multi-agent --- heads compete and coordinate --- yet we train it as
  if it were a monolithic optimizer.
  We formalize this gap: cross-entropy training induces an implicit potential game among heads, and gradient
  descent converges to Nash equilibria with potentially unbounded inefficiency due to unpriced
  externalities (redundancy, correlated errors).
  Our main result bounds the Price of Anarchy by $\Gamma(G)$, the off-diagonal mass of a head interaction
  matrix capturing weight and gradient coupling.
  Under mild smoothness assumptions, we prove that both \emph{excess hallucination probability} and
  \emph{excess head redundancy} scale with PoA, unifying two distinct failure modes into a single mechanism.
  The bound is prescriptive: regularization that reduces $\Gamma(G)$ provably tightens PoA.
  We instantiate this as GAME-LoRA, combining Barlow Twins decorrelation with log-determinant coordination pressure.
  Experiments validate the theory: $\Gamma(G)$ predicts hallucination ($p{<}0.05$), emergent coalitions
  exhibit selective coordination, and GAME-LoRA achieves up to 18\% hallucination reduction (8\% average)
  with no knowledge degradation --- a Pareto improvement inaccessible to methods ignoring the game structure.
\end{abstract}

%%%%%%%%%%%%%%%%%%%%%%%%%%%%%%%%%%%%%%%%%%%%%%%%%%%%%%%%%%%%%%%%%%%%%%%%%%%%%%%
% INTRODUCTION
%%%%%%%%%%%%%%%%%%%%%%%%%%%%%%%%%%%%%%%%%%%%%%%%%%%%%%%%%%%%%%%%%%%%%%%%%%%%%%%
\section{Introduction}
\begin{figure}[t]
  \centering
  \includegraphics[width=\columnwidth]{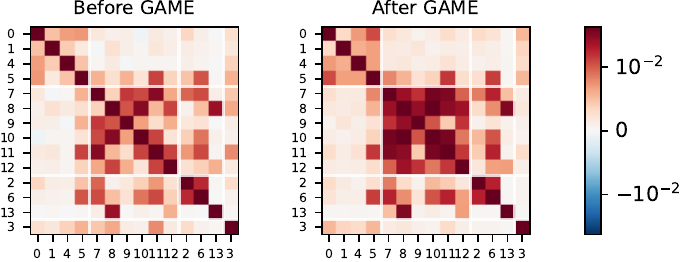}
  \caption{
    \textbf{Attention heads form strategic coalitions when trained with game-aware regularization.}
    We model multi-head attention as a multi-player game where heads compete for gradient credit while imposing
    unpriced externalities (redundancy, correlated errors) on each other.
    Standard cross-entropy training finds inefficient Nash equilibria; our method, GAME-LoRA, internalizes
    these externalities via Pigouvian taxes (log-determinant for compression, Barlow Twins for redundancy).
    The figure shows the head interaction matrix $G \in \mathbb{R}^{16 \times 16}$ for layer 19 of
    Qwen2.5-0.5B, reordered by spectral biclustering. \textbf{Left:} Before GAME-LoRA: heads show diffuse,
    unstructured coupling. \textbf{Right:} After GAME-LoRA: distinct coalitions emerge (white lines),
    with strengthened \emph{intra}-coalition coordination and \emph{inter}-coalition competition. This coalitional
    reorganization reduces the Price of Anarchy (\autoref{thm:poa}), achieving best-in-class hallucination
    performance among five methods while preserving knowledge (\autoref{tab:main}).
  }
  \label{fig:coalitions_hero}
\end{figure}

\begin{figure}[t]
  \centering
  \includegraphics[width=\columnwidth]{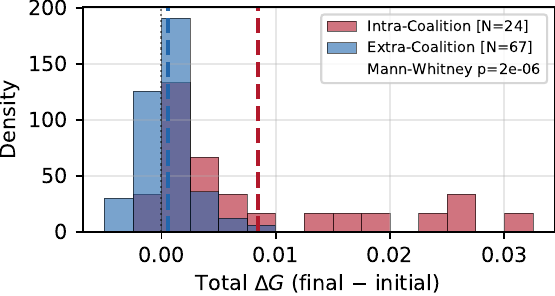}
  \caption{
    \textbf{GAME-LoRA induces selective coordination, not uniform decorrelation.}
    A naive interpretation of regularization losses would predict uniform weakening of all head interactions.
    Instead, we observe \emph{strategic differentiation}: heads selectively strengthen coordination within
    emergent coalitions while reducing coupling across coalition boundaries.
    The histogram shows changes in pairwise coupling $\Delta G_{ij}$ (final $-$ initial) for head pairs
    identified as intra-coalition (red, $N{=}24$) versus extra-coalition (blue, $N{=}67$) by spectral
    biclustering.
    Intra-coalition pairs strengthen significantly more (mean $\Delta G = 0.010$, dashed red) than
    extra-coalition pairs (mean $\Delta G \approx 0$, dashed blue), with $p = 2{\times}10^{-6}$ (Mann-Whitney).
    This selective reorganization mirrors coalition formation in cooperative game theory: heads that benefit
    from joint coordination form stable alliances, while heads serving distinct functions diversify---precisely
    the structure predicted by our externality-aware game formulation.
  }
  \label{fig:coalitions_support}
\end{figure}

Transformer attention heads are often described as learning specialized ``circuits'' that collaborate to solve tasks.
But what governs this collaboration?
Standard training optimizes a single loss (cross-entropy), treating the model as a monolithic function approximator.
Yet internally, attention heads interact: they share the residual stream, compete for gradient signal, and
influence each other's representations.
This gap---between multi-agent internals and single-agent training---has consequences.

We argue that standard training induces an \emph{implicit game} among attention heads.
Each head's ``strategy'' is its learned representation; its ``payoff'' is its contribution to reducing task loss.
Gradient descent finds Nash equilibria of this game.
The problem: these equilibria have \emph{unpriced externalities}.
When head $i$ duplicates head $j$, it wastes shared capacity (redundancy externality).
When head $i$'s errors correlate with head $j$'s, ensemble variance grows (hallucination externality).
Neither cost is reflected in the implicit payoff, so gradient descent finds equilibria with redundant heads,
dead neurons, and elevated tail risk.

Our contribution is to make the game explicit:

\begin{itemize}
  \item \textbf{A Game Theory of Attention.} We formalize multi-head attention as an implicit potential game
    (MultiHeadCE)---a game where individual gradients align with a shared objective---identify missing
    externality terms, and define an explicit externality-internalized game (MultiHeadPGAC).

  \item \textbf{Theoretical Foundations.} Under mild regularity assumptions, we derive Price of Anarchy
    bounds showing that hallucination probability and head redundancy are both controlled by the off-diagonal
    mass of the head interaction matrix (\autoref{thm:poa}, Corollaries~\ref{cor:hallucination}--\ref{cor:free-riding}).

  \item \textbf{Practical Method.} GAME-LoRA instantiates MultiHeadPGAC in modern transformer models with lightweight
    LoRA adapters. It achieves best-in-class hallucination reduction (5/6 benchmark wins, 80\% better than the
    runner-up) while preserving knowledge---not a tradeoff, but synergy.
\end{itemize}

\begin{table*}[t]
\centering
\caption{
        \textbf{Pricing externalities expands the Pareto frontier.}
        Comparison on six hallucination benchmarks (HaluEval, MemoTrap, TruthfulQA) and five knowledge benchmarks
        (MMLU, NQ, PopQA, WikiText BPB, Winogrande). Baseline is LoRA with vanilla CE loss.
        Training-time baselines include DEACON, Disagreement, and ME (Maximizing Entropy);
        inference-time baselines include CAD and ActDec.
        \emph{Key findings:} (1)~GAME-LoRA achieves +8.1\% hallucination improvement, 80\% better than
        the next-best method (CAD, +4.5\%), validating that Nash equilibrium selection via externality
        pricing yields superior solutions. (2)~Unlike ActDec ($-$2.8\% knowledge) and ME ($-$4.2\% knowledge),
        GAME-LoRA preserves knowledge ($-$0.1\%), confirming the Pareto improvement predicted by \autoref{cor:simultaneous}.
        Results averaged over 3 seeds; best in \textbf{bold}; WikiText is BPB (lower is better).
        }
\small
\begin{tabular}{llcccccc}
\toprule
& & Baseline & GAME-LoRA & CAD & Disagreement & ActDec & ME \\
\midrule
\multicolumn{8}{l}{\textit{Hallucination}} \\
  & HE-Dial & 0.458 & \textbf{0.491} & 0.479 & 0.471 & 0.458 & 0.455 \\
  & HE-QA & 0.376 & \textbf{0.445} & 0.417 & 0.391 & 0.376 & 0.395 \\
  & HE-Summ & 0.438 & \textbf{0.500} & 0.494 & 0.458 & 0.463 & 0.445 \\
  & MemoTrap & 0.642 & \textbf{0.650} & 0.641 & 0.641 & 0.642 & 0.642 \\
  & TFQA-MC1 & 0.252 & \textbf{0.263} & 0.255 & 0.247 & 0.251 & 0.252 \\
  & TFQA-MC2 & 0.401 & 0.412 & 0.392 & 0.399 & 0.416 & \textbf{0.419} \\
\midrule
\multicolumn{8}{l}{\textit{Knowledge}} \\
  & MMLU & 0.477 & 0.469 & \textbf{0.479} & 0.472 & 0.476 & 0.471 \\
  & NQ & 0.066 & \textbf{0.067} & 0.066 & 0.063 & \textbf{0.067} & 0.057 \\
  & PopQA & 0.111 & \textbf{0.112} & \textbf{0.112} & 0.111 & \textbf{0.112} & 0.110 \\
  & WikiText & 0.784 & 0.786 & 0.779 & \textbf{0.777} & 0.923 & 0.825 \\
  & Winogrande & 0.573 & 0.565 & 0.569 & \textbf{0.580} & 0.575 & 0.573 \\
\midrule
\multicolumn{8}{l}{\emph{Overall}} \\
  & \textit{Hallucination} & -- & \textbf{+8.0\%} & +4.5\% & +1.5\% & +1.5\% & +1.8\% \\
  & \textit{Knowledge} & -- & -0.1\% & \textbf{+0.4\%} & -0.6\% & -2.8\% & -4.2\% \\
\bottomrule
\label{tab:main}
\end{tabular}
\end{table*}

%%%%%%%%%%%%%%%%%%%%%%%%%%%%%%%%%%%%%%%%%%%%%%%%%%%%%%%%%%%%%%%%%%%%%%%%%%%%%%%
% A GAME THEORY OF ATTENTION
%%%%%%%%%%%%%%%%%%%%%%%%%%%%%%%%%%%%%%%%%%%%%%%%%%%%%%%%%%%%%%%%%%%%%%%%%%%%%%%
\section{A Game Theory of Attention}
\label{sec:theory}

Why do transformers hallucinate? We propose that one source of hallucination is coordination failure among
attention heads. Each head optimizes a shared loss, yet receives no signal about the costs it imposes on
others---duplicating information already captured elsewhere, or making errors that compound with correlated
errors from other heads. This is the classic structure of a game with externalities: individual rationality
diverges from collective optimality.

We formalize this intuition by modeling each attention head as a strategic agent.
Head $i$ has parameters $\theta_i$, produces output $h_i(x) \in \mathbb{R}^{d_h}$ (the attention-weighted
value after the head-level computation but before the output projection $W_O$), and contributes to the
shared residual stream via $W_O^{(i)} h_i(x)$.
The key question: how do heads interact, and when do those interactions lead to inefficiency?

\subsection{Preliminaries}

\paragraph{Notation.}
Let $X \sim \mathcal{D}$ denote inputs. The oracle truth is a probability vector
$y_\star(x) \in \Delta^{d-1}$, inducing a label random variable
$Y \mid X=x \sim y_\star(x)$.
A model (with parameters $w$) induces a stochastic encoder $p_w(z_{1:H}\mid x)$ over
$H$ head streams $Z_1,\dots,Z_H$, and a decoder $q_w(y\mid z_{1:H})$.

\paragraph{Hallucination.}
Define the model prediction $\hat y_w(x)\in\Delta^{d-1}$ and the total-variation deviation
\begin{align}
  E_w(x)
  &\triangleq
  \frac{1}{2}\big\|\hat y_w(x)-y_\star(x)\big\|_1, \\
  H_\delta &\triangleq \{x: E_w(x)\ge \delta\}.
\end{align}
The set $H_\delta$ captures inputs where the model's prediction deviates substantially from ground
truth---the ``hallucination regime.'' Our goal is to bound the probability mass on $H_\delta$.

\paragraph{Redundancy.}
Define the conditional total correlation (multi-information) across streams
\begin{align}
  \mathrm{TC}(Z_{1:H}\mid X)
  \;\triangleq\;
  \sum_{i=1}^H H(Z_i\mid X)-H(Z_{1:H}\mid X).
\end{align}
High TC means heads encode overlapping information given the input---a form of wasted capacity. When heads
are conditionally independent given $X$, we have $\mathrm{TC}=0$; when they are perfectly correlated,
TC is maximal.

\paragraph{Interaction matrix.}
We define a head interaction matrix that captures both \emph{structural} coupling (through output projection
weights) and \emph{gradient} coupling (through backpropagated learning signal).

\begin{definition}[Weight Coupling]\label{def:weight_coupling}
  For heads $i, j \in [H]$ with output projections $W_O^{(i)}, W_O^{(j)} \in \mathbb{R}^{d \times d_h}$, define
  \begin{align}
    \omega_{ij} \triangleq \frac{\langle W_O^{(i)}, W_O^{(j)} \rangle_F}{\|W_O^{(i)}\|_F \|W_O^{(j)}\|_F}.
  \end{align}
\end{definition}

\begin{definition}[Gradient Coupling]\label{def:gradient_coupling}
  Let $\eta = \nabla_\ell \mathcal{L}_{\mathrm{CE}}$ denote the gradient of cross-entropy loss with respect
  to logits. Define the backpropagated gradient through head $i$ as $g_i = (W_O^{(i)})^\top \eta \in
  \mathbb{R}^{d_h}$, and
  \begin{align}
    \rho_{ij} \triangleq \frac{\langle g_i, g_j \rangle}{\|g_i\| \|g_j\|}.
  \end{align}
\end{definition}

\begin{definition}[Head Interaction Matrix]\label{def:interaction_matrix}
  The head interaction matrix $G \in \mathbb{R}^{H \times H}$ has entries
  \begin{align}
    G_{ij} \triangleq \omega_{ij} \cdot \rho_{ij},
  \end{align}
  with interaction strength $\Gamma(G) \triangleq \|G - I\|_F$.
\end{definition}

The product form $G_{ij} = \omega_{ij} \cdot \rho_{ij}$ captures the key insight that externalities arise
only when heads are \emph{both} structurally aligned (projecting to similar subspaces) \emph{and} receiving
correlated gradient signal. If either coupling vanishes, the pair contributes no inefficiency. We use $G$ and
$\Gamma(G)$ as tractable proxies for the intractable $\mathrm{TC}(Z_{1:H} \mid X)$: under Gaussian
assumptions, high off-diagonal mass in $G$ implies high pairwise mutual information, which lower-bounds TC.

\begin{definition}[Information Bottleneck Social Objective]
  \label{def:ib_social}
  Define the social objective via an information-bottleneck (IB) Lagrangian:
  \begin{align}
    C^\star_{\mathrm{IB}}(w)
    \;\triangleq\;
    &\underbrace{\mathbb{E}\big[-\log q_w(Y\mid Z_{1:H})\big]}_{\text{distortion}}\\
    &\;+\; \beta_{R}\underbrace{\mathrm{TC}(Z_{1:H}\mid X)}_{\text{redundancy}}\\
    &\;+\; \beta_{C}\underbrace{\sum_{i=1}^H I(Z_i;X)}_{\text{compression}}.
  \end{align}
  All expectations are with respect to the joint distribution induced by
  $X\sim\mathcal{D}$, $Z_{1:H}\sim p_w(\cdot\mid X)$, and $Y\sim y_\star(X)$.
\end{definition}

This objective is derived directly from the multi-view Information Bottleneck \cite{tishby2000information}:
each head should extract a minimal sufficient statistic of the input for predicting the target, while
avoiding redundancy with other heads. A socially optimal configuration minimizes $C^\star_{\mathrm{IB}}$. The
question is whether decentralized gradient descent---where each head greedily minimizes its own loss---finds
such configurations.

\begin{definition}[Information Bottleneck Price of Anarchy]
  \label{def:poa}
  For a game $\mathcal{G}$ with equilibrium set $\mathrm{NE}(\mathcal{G})$,
  define
  \begin{align}
    \mathrm{PoA}(\mathcal{G})
    \;\triangleq\;
    \frac{\max_{w\in \mathrm{NE}(\mathcal{G})} C^\star_{\mathrm{IB}}(w)}
    {\min_{w} C^\star_{\mathrm{IB}}(w)}.
  \end{align}
\end{definition}

The Price of Anarchy measures how much worse the worst equilibrium is compared to the social optimum. A PoA
of 1 means all equilibria are optimal; larger values indicate coordination failure. Our main results show
that (i) standard cross-entropy training has potentially large PoA, and (ii) adding appropriate
regularization reduces PoA, which in turn bounds hallucination probability.

\subsection{MultiHeadCE: An Implicit, Incomplete Game}
\label{sec:implicit_game}

Standard transformer training minimizes cross-entropy loss via gradient descent. We show this implicitly
defines a game among attention heads---but one where externalities go unpriced.

\begin{definition}[MultiHeadCE]
  \label{def:multiheadce}
  Partition parameters as $w=(\theta_1,\dots,\theta_H,\theta_{\mathrm{rest}})$, where
  $\theta_i$ are head-specific parameters (player actions). Fix credit shares
  $\pi_i>0$---each head's apportioned responsibility for the collective loss---with
  $\sum_{i=1}^H \pi_i=1$. Define private costs (the loss each head individually minimizes) for $i\in[H]$
  \begin{align}
    C_i^{\mathrm{CE}}(w)
    \;\triangleq\;
    \pi_i\,\mathbb{E}\big[-\log q_w(Y\mid Z_{1:H})\big]
    \;+\;\frac{\alpha}{2}\|\theta_i\|_2^2.
  \end{align}
\end{definition}

\begin{theorem}[MultiHeadCE is a Weighted Potential Game]
  \label{thm:mhce_potential}
  Given \autoref{def:multiheadce}, define the potential
  \begin{align}
    \Phi_{\mathrm{CE}}(w)
    \;\triangleq\;
    &\mathbb{E}\big[-\log q_w(Y\mid Z_{1:H})\big]\\
    &\;+\; \frac{\alpha}{2}\sum_{i=1}^H \frac{1}{\pi_i}\|\theta_i\|_2^2.
  \end{align}
  Then MultiHeadCE is a weighted potential game in the sense that, for each $i$,
  \begin{align}
    \nabla_{\theta_i} C_i^{\mathrm{CE}}(w)
    \;=\;
    \pi_i \nabla_{\theta_i}\Phi_{\mathrm{CE}}(w).
  \end{align}
  Consequently, any point $w$ satisfying $\nabla_{\theta_i}C_i^{\mathrm{CE}}(w)=0$ for all $i$
  is a first-order (local) Nash equilibrium. Moreover, if $\Phi_{\mathrm{CE}}$ has $L_\Phi$-Lipschitz
  gradient, then gradient descent on $\Phi_{\mathrm{CE}}$ with step size $\eta\in(0,1/L_\Phi)$ produces a
  sequence $\{w^t\}$ with $\min_{0\le t<T}\|\nabla \Phi_{\mathrm{CE}}(w^t)\|_2^2 \to 0$ as $T\to\infty$.
\end{theorem}

\begin{proof}[Proof sketch]
  Differentiate $C_i^{\mathrm{CE}}$ and $\Phi_{\mathrm{CE}}$ with respect to $\theta_i$: both yield the same
  gradient direction (the shared CE term) plus a regularization term, with the $\pi_i$ weighting ensuring
  $\nabla_{\theta_i} C_i = \pi_i \nabla_{\theta_i} \Phi$. Thus stationary points of $\Phi$ are Nash equilibria.
  Convergence follows from the standard descent lemma for smooth functions.
\end{proof}

\autoref{thm:mhce_potential} guarantees convergence to \emph{some} Nash equilibrium, but says nothing about
\emph{which} one. The problem: MultiHeadCE has \emph{unpriced externalities}. When head $i$ duplicates head
$j$, it wastes shared capacity (redundancy).
When head $i$'s errors correlate with head $j$'s, ensemble variance grows and tail risk compounds
(hallucination).
Neither cost is reflected in $C_i^{\mathrm{CE}}$, so gradient descent finds equilibria with redundant heads,
dead neurons, and elevated hallucination.

\subsection{MultiHeadPGAC: An Explicit Public-Goods + Anti-Coordination Game}
To design the right incentives, we ask: what game are heads implicitly playing? The externalities above have
recognizable structure: Redundancy arises when heads free-ride instead of differentiating; coverage gaps when
none invests in minority features. These are the signature pathologies of \emph{public-goods games} (under-provision)
and \emph{anti-coordination games} (failure to specialize). We posit CE training induces a PGAC game with
missing incentives, and augment utilities with charges that restore them.

\begin{definition}[MultiHeadPGAC]\label{def:multiheadpgac}
  Define MultiHeadPGAC by augmenting MultiHeadCE with IB-externality charges:
  \begin{align}
    C_i^{\mathrm{PGAC}}(w)
    \;\triangleq\;
    &\pi_i\,\mathbb{E}\big[-\log q_w(Y\mid Z_{1:H})\big]\\
    &+\frac{\alpha}{2}\|\theta_i\|_2^2 + \beta_C\,\tau_i^{C}(w)\\
    &+\beta_{R}\,\tau_i^{R}(w),
  \end{align}
  where $\tau_i^{C}$ and $\tau_i^{R}$ approximate the marginal contribution
  of player $i$ to $\sum_j I(Z_j;X)$ and $\mathrm{TC}(Z_{1:H}\mid X)$ respectively.
\end{definition}

We require two mild regularity conditions on the loss landscape and model parameterization.

\begin{assumption}[Lipschitz Gradient]\label{assm:lipschitz}
  The cross-entropy loss $\mathcal{L}_{\mathrm{CE}}$ has $L$-Lipschitz gradient in the output projections
  $\theta = (W_O^{(1)}, \ldots, W_O^{(H)})$:
  \begin{align}
    \|\nabla_\theta \mathcal{L}_{\mathrm{CE}}(\theta) - \nabla_\theta \mathcal{L}_{\mathrm{CE}}(\theta')\|
    \le L \|\theta - \theta'\|.
  \end{align}
\end{assumption}

\begin{assumption}[Bounded Projections]\label{assm:bounded}
  There exists $B > 0$ such that $\|W_O^{(i)}\|_F \le B$ for all $i \in [H]$.
\end{assumption}

\autoref{assm:lipschitz} is standard for smooth optimization and holds for softmax cross-entropy with bounded
logits. \autoref{assm:bounded} is trivially satisfied in practice via weight decay or explicit clipping, and
implies bounded attention outputs $\|W_O^{(i)} a_i\| \le B \|a_i\|$.

\begin{theorem}[Multi-Head Games Have Price of Anarchy]\label{thm:poa}
  Under Assumptions~\ref{assm:lipschitz}--\ref{assm:bounded}, for any Nash equilibrium $\theta^{\mathrm{NE}}$
  with interaction matrix $G = G(\theta^{\mathrm{NE}})$ satisfying $\Gamma(G)^2 < \alpha/L$:
  \begin{align}
    \mathrm{PoA}(\mathrm{PGAC})
    \;\triangleq\;
    \frac{C^\star_{\mathrm{IB}}(\theta^{\mathrm{NE}})}{\min_\theta C^\star_{\mathrm{IB}}(\theta)}
    \;\le\;
    \frac{1 + \beta_R + \beta_C}{1 - \frac{L}{\alpha}\Gamma(G)^2}
  \end{align}
  where $\Gamma(G)^2 = 2\sum_{i < j} \omega_{ij}^2 \rho_{ij}^2$.
\end{theorem}

\begin{proof}[Proof sketch]
  The key insight is that $\Gamma(G)$ controls cross-head feedback: under
  Assumptions~\ref{assm:lipschitz}--\ref{assm:bounded}, the off-diagonal Hessian blocks satisfy
  $\|\nabla^2_{\theta_i\theta_j} D\| \le L|G_{ij}|$, so head correlations amplify gradient perturbations
  along any path from equilibrium to optimum.
  This yields a smoothness inequality: the sum of costs under unilateral deviations to optimum is bounded by
  $(1+\beta_R+\beta_C)C^\star_{\mathrm{IB}}(\theta^\star)$ plus a feedback term
  $\tfrac{L}{\alpha}\Gamma(G)^2 C^\star_{\mathrm{IB}}(\theta^{\mathrm{NE}})$.
  Nash optimality then implies equilibrium cost is at most this sum; rearranging under $\Gamma(G)^2 <
  \alpha/L$ isolates $C^\star_{\mathrm{IB}}(\theta^{\mathrm{NE}})$ and produces the bound.
  Full proof in \autoref{app:poa_proof}.
\end{proof}

\autoref{thm:poa} is the theoretical engine of this paper: it converts a measurable property of head
representations ($\Gamma(G)$, the off-diagonal mass of the interaction matrix) into a bound on equilibrium
inefficiency. The bound is \emph{prescriptive}: reducing $\Gamma(G)$ tightens PoA, providing theoretical
motivation for decorrelation-based regularization. High $\Gamma(G)$ signals coordination failure: correlated
heads free-ride on each other's signal (redundancy, \autoref{cor:free-riding}) and amplify each other's
errors into the tail (hallucination, \autoref{cor:hallucination}).

\paragraph{Subspace optimization.}
Our bounds depend on $\Gamma(G)$, a function of activations, not parameters, and thus apply
equally to full gradient descent and subspace methods like LoRA. Moreover, restricting
strategy spaces can only reduce PoA: the worst-case equilibrium in a constrained game
cannot exceed that in the full game.

\begin{figure}[t]
  \centering
  \includegraphics[width=\columnwidth]{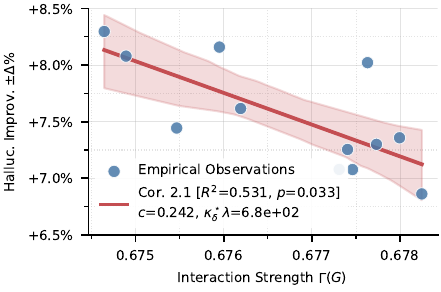}
  \caption{
    \textbf{Instantiation of \autoref{cor:hallucination}.}
    Since excess hallucination $\propto \mathrm{PoA} \propto 1/(1 - c \cdot \Gamma)$,
    we fit $\Delta H = a - \lambda/(1 - c \cdot \Gamma)$ to training trajectories.
    The fit achieves $R^2 = 0.53$ ($p < 0.05$) with bootstrap 95\% prediction bands (shaded).
    \textbf{Non-vacuousness:} The fitted $c \cdot \Gamma \approx 0.2 \ll 1$ confirms the bound
    is finite and tight---reducing $\Gamma$ yields proportional hallucination reduction.
  }
  \label{fig:hallucination}
\end{figure}

\begin{corollary}[Excess Hallucination Is Price of Anarchy]
  \label{cor:hallucination}
  Let $w^\star \in \arg\min_{w} C^\star_{\mathrm{IB}}(w)$ denote the social optimum and assume
  $\Pr(H_\delta\mid w^\star) > 0$. Define the constant $\kappa^\star_\delta \triangleq
  C^\star_{\mathrm{IB}}(w^\star) / (2\delta^2 \Pr(H_\delta\mid w^\star))$, which depends only on $w^\star$.
  Then, for any equilibrium $w^{\mathrm{NE}}\in\mathrm{NE}(\mathcal{G})$ of
  $\mathcal{G}\in\{\mathrm{CE},\mathrm{PGAC}\}$,
  \begin{align}
    \frac{\Pr(H_\delta\mid w^{\mathrm{NE}})}{\Pr(H_\delta\mid w^\star)}
    &\le \kappa^\star_\delta \cdot \mathrm{PoA}(\mathcal{G}).
    \label{eq:hallucination_excess}
  \end{align}
\end{corollary}

\begin{proof}[Proof sketch]
  For any $w$, the chain Markov $\to$ Pinsker $\to$ data-processing yields
  $\Pr(H_\delta \mid w) \le C^\star_{\mathrm{IB}}(w)/(2\delta^2)$.
  Apply at $w^{\mathrm{NE}}$, divide by $\Pr(H_\delta \mid w^\star)$, and invoke the PoA definition:
  \begin{align*}
    \frac{\Pr(H_\delta \mid w^{\mathrm{NE}})}{\Pr(H_\delta \mid w^\star)}
    \le \frac{C^\star_{\mathrm{IB}}(w^{\mathrm{NE}})}{2\delta^2 \Pr(H_\delta \mid w^\star)}
    \le {\kappa^\star_\delta} \cdot \mathrm{PoA}.
  \end{align*}
  As $\kappa^\star_\delta$ depends only on the social optimum, PoA is the \emph{sole} equilibrium-dependent
  factor: reducing $\Gamma(G)$ tightens PoA (\autoref{thm:poa}), directly compressing excess hallucination.
\end{proof}

This is our first concrete result: an \emph{excess} hallucination bound where $\kappa^\star_\delta$ is a
fixed problem constant and PoA is the only lever we control. Whatever irreducible hallucination the task
and architecture admit, equilibrium can be at most $\kappa^\star_\delta \cdot \mathrm{PoA}$ times worse --- so
\emph{reducing PoA is the entire game}.

\paragraph{Instantiation and non-vacuousness.}
\autoref{fig:hallucination} validates \autoref{cor:hallucination} by fitting its functional form: excess
hallucination $\propto$ PoA $\propto 1/(1 - c\Gamma)$. The fit achieves $R^2 = 0.53$ ($p < 0.05$), confirming
the predicted relationship. Crucially, we verify \textbf{non-vacuousness}: the fitted coupling satisfies $c
\cdot \Gamma \approx 0.2 \ll 1$ across all training runs, ensuring $(1 - c \cdot \Gamma) \approx 0.8 > 0$
remains bounded away from zero. The bootstrap bands remain tight despite parameter uncertainty, demonstrating
that the functional form is robustly identified. Next, we show PoA also controls capacity waste.

\begin{definition}[Information-Theoretic Free-Riding]
  \label{def:free_riding}
  Fix any ordering of heads. For $\tau>0$, define the free-rider set
  \begin{align}
    \mathrm{FR}_\tau(w)
    \;\triangleq\;
    \Big\{ i :\ \mathbb{E}\big[I(Z_i;Z_{<i}\mid X)\big] \ge \tau \Big\}.
  \end{align}
\end{definition}

A head ``free-rides'' when it copies information already encoded by earlier heads. This wastes capacity and
is a symptom of coordination failure.

\begin{corollary}[Excess Free-Riding Is Price of Anarchy]\label{cor:free-riding}
  Fix $\tau>0$ and let $\mathrm{FR}_\tau(w)$ be the free-rider set from \autoref{def:free_riding}.
  Let $w^\star \in \arg\min_{w} C^\star_{\mathrm{IB}}(w)$ denote the social optimum and assume
  $r^\star_\tau \triangleq |\mathrm{FR}_\tau(w^\star)| > 0$. Define the constant $\kappa^\star_\tau
  \triangleq C^\star_{\mathrm{IB}}(w^\star)/(\beta_R \tau \, r^\star_\tau)$, which depends only on $w^\star$.
  Then, for any equilibrium $w^{\mathrm{NE}}\in\mathrm{NE}(\mathcal{G})$ of
  $\mathcal{G}\in\{\mathrm{CE},\mathrm{PGAC}\}$,
  \begin{align}
    \frac{|\mathrm{FR}_\tau(w^{\mathrm{NE}})|}{|\mathrm{FR}_\tau(w^\star)|}
    &\le \kappa^\star_\tau \cdot \mathrm{PoA}(\mathcal{G}).
  \end{align}
\end{corollary}

\begin{proof}[Proof sketch]
  Markov and the entropy chain rule yield the absolute bound
  $|\mathrm{FR}_\tau(w)| \le C^\star_{\mathrm{IB}}(w)/(\beta_R \tau)$ for any $w$
  (\autoref{app:free_riding_proof}).
  Apply at $w^{\mathrm{NE}}$, divide by $r^\star_\tau$, and invoke PoA:
  \begin{align*}
    \frac{|\mathrm{FR}_\tau(w^{\mathrm{NE}})|}{r^\star_\tau}
    \le \frac{C^\star_{\mathrm{IB}}(w^{\mathrm{NE}})}{\beta_R \tau \, r^\star_\tau}
    \le \mathrm{PoA} \cdot \underbrace{\frac{C^\star_{\mathrm{IB}}(w^\star)}{\beta_R \tau \,
    r^\star_\tau}}_{\kappa^\star_\tau}.
  \end{align*}
  As with \autoref{cor:hallucination}, $\kappa^\star_\tau$ depends only on the social optimum;
  PoA is the sole equilibrium-dependent factor.
\end{proof}

This is our second concrete result: an \emph{excess} free-riding bound paralleling \autoref{cor:hallucination}.
The problem constant $\kappa^\star_\tau$ is fixed; PoA is again the only lever.
Together, both corollaries show that PoA controls \emph{both} failure modes---tail risk and capacity waste---so
reducing $\Gamma(G)$ via regularization improves reliability and efficiency simultaneously.

\begin{corollary}[Efficiency Is Reliability]\label{cor:simultaneous}
  Fix $\delta\in(0,1]$ and $\tau>0$.
  For $\mathcal{G}\in\{\mathrm{CE},\mathrm{PGAC}\}$ define worst-case equilibrium reliability metrics
  \begin{align}
    \mathcal{R}_{\mathrm{hall}}(\mathcal{G};\delta)
    &\triangleq
    \max_{w^{\mathrm{NE}}\in\mathrm{NE}(\mathcal{G})} \Pr(H_\delta\mid w^{\mathrm{NE}}),
    \\
    \mathcal{R}_{\mathrm{fr}}(\mathcal{G};\tau)
    &\triangleq
    \max_{w^{\mathrm{NE}}\in\mathrm{NE}(\mathcal{G})} |\mathrm{FR}_\tau(w^{\mathrm{NE}})|.
  \end{align}
  Then, for each $\mathcal{G}$,
  \begin{align}
    \mathcal{R}_{\mathrm{hall}}(\mathcal{G};\delta)
    &\;\le\;
    \frac{\mathrm{PoA}(\mathcal{G})}{2\delta^2}\cdot \min_w C^\star_{\mathrm{IB}}(w),\\
    \mathcal{R}_{\mathrm{fr}}(\mathcal{G};\tau)
    &\;\le\;
    \frac{\mathrm{PoA}(\mathcal{G})}{\beta_R\tau}\cdot \min_w C^\star_{\mathrm{IB}}(w).
  \end{align}
  Consequently, holding $\min_w C^\star_{\mathrm{IB}}(w)$ fixed, any reduction in $\mathrm{PoA}(\mathcal{G})$
  yields uniformly tighter \emph{bounds} on hallucination incidence and redundancy-driven free-riding at
  equilibrium (under the stated assumptions).
\end{corollary}

\begin{proof}[Proof sketch]
  Direct combination of \autoref{cor:hallucination} and \autoref{cor:free-riding}: take maxima over equilibria
  on both sides of each bound. Both depend on the \emph{same} $\mathrm{PoA}(\mathcal{G})$ term, which in turn
  depends on the same $\Gamma(G)$. Any intervention that reduces $\Gamma(G)$---such as Barlow Twins
  regularization---therefore tightens both bounds simultaneously.
\end{proof}

\autoref{cor:simultaneous} packages together our final theoretical contribution: it shows
that capacity and
reliability can improve \emph{simultaneously} at fixed parameter count. This breaks the conventional wisdom
--- rooted in calibration-accuracy tradeoffs \cite{guo2017calibration} and the ``alignment tax'' of RLHF
\cite{lin2024alignment} --- that reliability costs capability. The mechanism is that both failure modes share a
common cause: high $\Gamma(G)$ from unpriced externalities. Regularizers that reduce $\Gamma(G)$ thus improve
\emph{both} metrics by moving to a lower-PoA equilibrium, not by trading one against the other.

%%%%%%%%%%%%%%%%%%%%%%%%%%%%%%%%%%%%%%%%%%%%%%%%%%%%%%%%%%%%%%%%%%%%%%%%%%%%%%%
% GAME-LORA
%%%%%%%%%%%%%%%%%%%%%%%%%%%%%%%%%%%%%%%%%%%%%%%%%%%%%%%%%%%%%%%%%%%%%%%%%%%%%%%
\section{GAME-LoRA: A Practical Demonstration}

Next, we demonstrate via GAME-LoRA that the MultiHeadPGAC game is practical and can be
implemented in modern transformers via lightweight LoRA adapters.

\subsection{Architecture}

\begin{figure}[ht]
  \centering
  \includegraphics[width=\columnwidth]{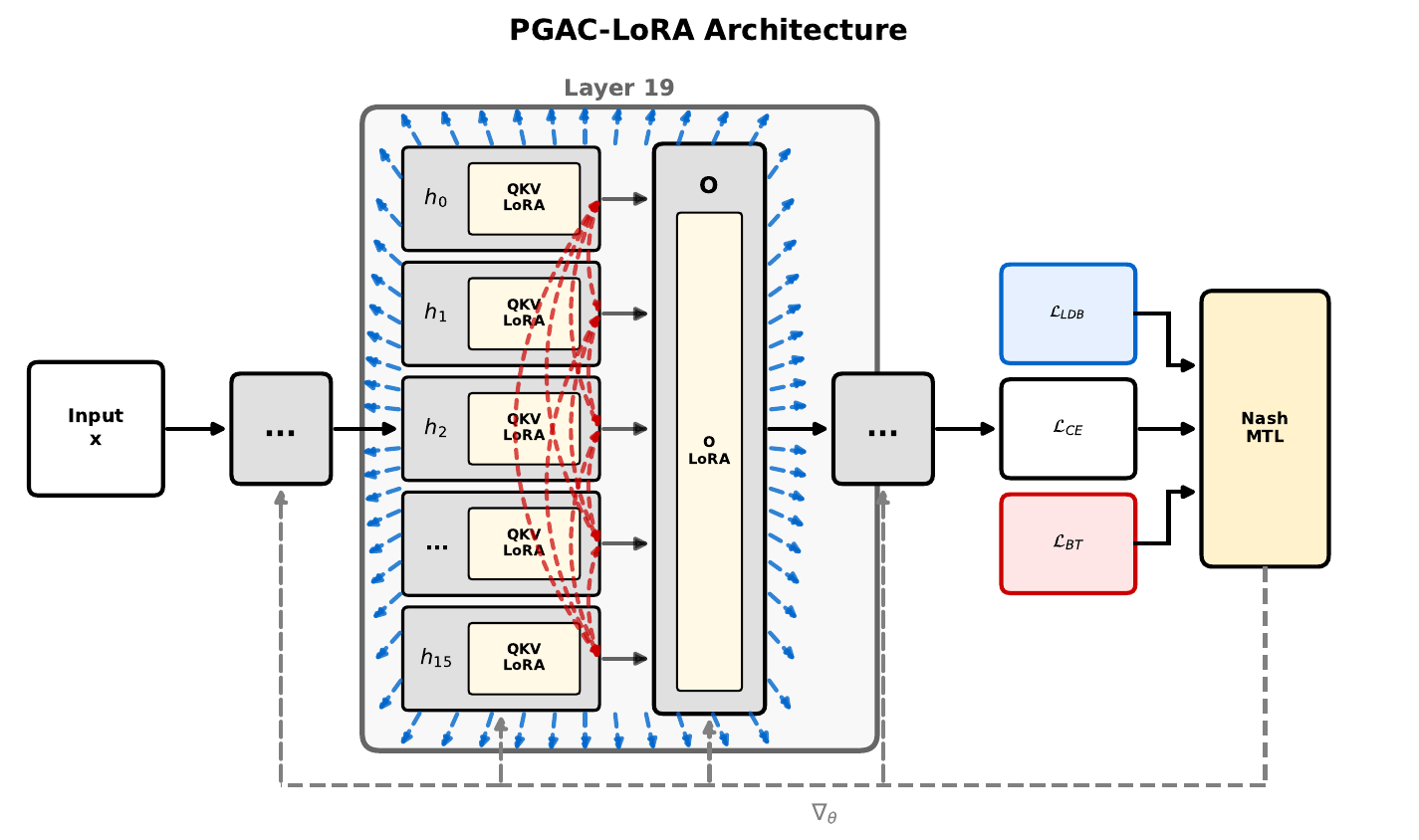}
  \caption{GAME-LoRA architecture. Layer 19 with 16 attention heads ($h_0$--$h_{15}$), each containing QKV
    LoRA adapters, plus O LoRA in the output projection. Red dashed arrows: $\mathcal{L}_{\mathrm{ABT}}$
    pairwise head decorrelation. Blue dashed arrows: $\mathcal{L}_{\mathrm{LDB}}$ expansion pressure on head
    Gram matrix. Three losses ($\mathcal{L}_{\mathrm{CE}}$, $\mathcal{L}_{\mathrm{LDB}}$,
  $\mathcal{L}_{\mathrm{ABT}}$) combined via Nash-MTL arbitration.}
  \label{fig:architecture}
\end{figure}

GAME-LoRA applies LoRA adapters to all attention projections (Q, K, V, O) across all layers, while
computing regularization losses at a single late \emph{design layer}.
Pretrained weights remain frozen; only adapter weights are updated.
The total loss combines cross-entropy with two regularizers:
\begin{equation}
  \mathcal{L} = \mathcal{L}_{\mathrm{CE}} + \lambda_{\mathrm{LDB}} \mathcal{L}_{\mathrm{LDB}} +
  \lambda_{\mathrm{ABT}} \mathcal{L}_{\mathrm{ABT}},
\end{equation}
with gradients arbitrated via Nash-MTL \cite{navon2022multi}. Full hyperparameters in \autoref{app:implementation}.

\subsection{Regularization Losses}

We instantiate the externality charges from \autoref{def:multiheadpgac} using two standard objectives
at different granularities.
The \textbf{log-determinant barrier} \citep{federici2020mvib} operates at \emph{head level}:
\begin{equation}
  \mathcal{L}_{\mathrm{LDB}} = -\log \det(G + \epsilon I),
\end{equation}
where $G \in \mathbb{R}^{H \times H}$ is the head interaction matrix (\autoref{def:interaction_matrix}).
The \textbf{Barlow Twins loss} \citep{zbontar2021barlow} operates at \emph{feature level}. For each head
pair $(i,j)$, let $\hat{C}_{ij} = \frac{1}{N} \tilde{O}_i^\top \tilde{O}_j \in \mathbb{R}^{d_h \times d_h}$
be the cross-correlation of z-scored head outputs. The loss is:
\begin{equation}
  \mathcal{L}_{\mathrm{ABT}} = \mathop{\mathbb{E}}_{i < j} \left[ w_{ij} \|\hat{C}_{ij} - I\|_F^2 \right],
\end{equation}
where $w_{ij} \triangleq \alpha + (1{-}\alpha) \cdot \mathrm{softplus}(-\beta (G_{ij} - \tau))$ adaptively
weights pairs. Note that $i < j$ sums over \emph{cross-head} pairs only; targeting $I$ (not zero) for
$\hat{C}_{ij}$ penalizes off-diagonal entries (feature $k$ of head $i$ correlating with feature $l \neq k$
of head $j$) while the diagonal constraint ($[\hat{C}_{ij}]_{kk} \to 1$) anchors a shared feature basis,
preventing spurious rotations that would otherwise confound the off-diagonal penalty
\citep{zbontar2021barlow}. This extends Barlow Twins from same-sample augmentation pairs to cross-head
decorrelation. Hyperparameters are in \autoref{app:implementation}.

%%%%%%%%%%%%%%%%%%%%%%%%%%%%%%%%%%%%%%%%%%%%%%%%%%%%%%%%%%%%%%%%%%%%%%%%%%%%%%%
% EXPERIMENTS
%%%%%%%%%%%%%%%%%%%%%%%%%%%%%%%%%%%%%%%%%%%%%%%%%%%%%%%%%%%%%%%%%%%%%%%%%%%%%%%
\section{Experimental Validation}

Here, we demonstrate at small scale that game-theoretic regularization can improve the reliability of modern
transformers across hallucination and knowledge benchmarks.

\subsection{Experimental Setup}
We fine-tune Qwen2.5-0.5B using LoRA on The Pile (20M tokens) and evaluate on six hallucination (HaluEval,
TruthfulQA, MemoTrap) and four knowledge benchmarks (NQ, PopQA, WikiText, WinoGrande). The ``Baseline'' column
reports parameter- and data-matched LoRA fine-tuned with CE loss only. Other baselines include inference-time
methods (CAD, ActDec) and training-time methods (Disagreement, ME). Full details in
\autoref{app:implementation}.

\subsection{Main Results}
\autoref{tab:main} presents results across all benchmarks.
GAME-LoRA achieves best-in-class hallucination reduction (+8.1\% overall), winning 5 of 6 individual
benchmarks and outperforming the next-best method by 80\% (CAD, +4.5\%).
Critically, GAME-LoRA is the \emph{only} method to improve MemoTrap (+1.2\%), a memorization-based
hallucination benchmark where no other method improves.
This suggests GAME-LoRA addresses a failure mode---over-reliance on memorized patterns---that other
methods cannot.

Unlike ActDec, which trades knowledge for hallucination improvement ($-$2.8\% knowledge), GAME-LoRA
achieves strong hallucination gains (+8.1\%) while preserving knowledge ($-$0.1\%).
This validates \autoref{cor:simultaneous}: internalizing \emph{both} coordination and competition
externalities expands the Pareto frontier, achieving what partial interventions cannot.

% \subsection{Why Baselines Fall Short}

% No baseline addresses the implicit game structure of multi-head attention (\autoref{sec:implicit_game}).

% \paragraph{Inference-time methods cannot reshape the equilibrium.}
% CAD and ActDec operate at decoding time on heads trained under MultiHeadCE---the implicit game with
% unpriced externalities.
% They cannot change which equilibrium the model occupies; heads remain correlated, free-riders remain inactive.
% CAD achieves strong hallucination gains (+3.0\%) but cannot match GAME-LoRA's MemoTrap improvement.
% ActDec's activation sharpening trades off sharply: $-$3.9\% knowledge, $-$17.7\% WikiText perplexity.
% Post-hoc corrections to suboptimal play cannot expand the frontier.

% \paragraph{Training-time baselines ignore the game.} Disagreement Regularization \emph{does} modify
% training, but does not recognize the game-theoretic structure. It implements heuristic decorrelation
% via cosine distance maximization without modeling the coordination-competition tradeoff that determines
% equilibrium quality. Critically, it does not improve MemoTrap ($-$0.2\%), suggesting that uniform
% decorrelation disrupts memorization-dependent circuits without replacing them.

% GAME-LoRA trains under MultiHeadPGAC, achieving a new equilibrium with lower $\Gamma(G)$ and larger active
% set $|A|$ by \emph{explicitly} pricing both externalities.
% The BT loss decorrelates (competition); the LDB loss maintains full-rank covariance (coordination).
% This dual mechanism is why GAME-LoRA uniquely improves MemoTrap while dominating other benchmarks.np

\subsection{Efficiency Is Reliability}

Prior work frames hallucination reduction and knowledge retention as competing objectives
\citep{zhang2024siren}.
Our results suggest this tradeoff is an artifact of incomplete interventions that address only one
externality.

Hallucination detection benefits from diversity; knowledge retrieval requires heads to \emph{coordinate}
on the correct answer; generation requires \emph{both}.
GAME-LoRA's +8.1\% hallucination improvement with preserved knowledge ($-$0.1\%) --- without the generation
collapse seen in ActDec (-17.7\% WikiText) --- demonstrates both are achievable: decorrelating error modes
(\autoref{cor:hallucination}) while preserving coordinated representations (\autoref{cor:free-riding}).
Finally, unlike CAD or ActDec, GAME-LoRA incurs \emph{zero} inference cost.

%%%%%%%%%%%%%%%%%%%%%%%%%%%%%%%%%%%%%%%%%%%%%%%%%%%%%%%%%%%%%%%%%%%%%%%%%%%%%%%
% MECHANISTIC ANALYSIS
%%%%%%%%%%%%%%%%%%%%%%%%%%%%%%%%%%%%%%%%%%%%%%%%%%%%%%%%%%%%%%%%%%%%%%%%%%%%%%%
\section{Mechanistic Analysis}
\begin{table}[t]
  \centering
  \caption{Ablation study: average $\pm\Delta$\% by category relative to Qwen2.5-0.5B baseline. BT = Barlow Twins, LDB = log-det barrier. Best per column in \textbf{bold}. Full details in Appendix~\ref{app:ablations}.}
  \small
  \label{tab:ablation_relative}
  \begin{tabular}{lcc}
    \toprule
    Method & Hallucination & Knowledge \\
    \midrule
    GAME-LoRA & \textbf{+8.0\%} & -0.1\% \\
    Baseline + BT & +5.1\% & -1.2\% \\
    Baseline + LDB & +1.7\% & \textbf{+4.9\%} \\
    Baseline + LDB w/o NashMTL & -0.1\% & -2.4\% \\
    \bottomrule
  \end{tabular}
\end{table}
\subsection{Ablations}

\autoref{tab:ablation_relative} shows that the full GAME-LoRA method achieves best hallucination (+8.0\%)
while preserving knowledge ($-$0.1\%). Individual components show distinct tradeoffs: BT alone achieves
+5.1\% hallucination but $-$1.2\% knowledge, while LDB alone achieves +1.7\% hallucination with +4.9\%
knowledge. The combination yields the best overall balance of hallucination reduction with knowledge parity.

\subsection{Mechanistic Signatures}

Four mechanistic signatures validate our game framework.

\paragraph{1. Selective coordination, not uniform decorrelation.}
Figures~\ref{fig:coalitions_hero} and ~\ref{fig:coalitions_support} reveal \emph{selective} restructuring:
intra-cluster coupling strengthens while inter-cluster coupling weakens (Mann-Whitney $p < 10^{-5}$).
Heads self-organize into coalitions that coordinate internally while differentiating externally---the
equilibrium structure predicted by public goods games with local competition.

\paragraph{2. $\Gamma(G)$ as causal predictor of hallucination.}
\autoref{cor:hallucination} predicts that hallucination probability scales with PoA, which scales with
interaction strength $\Gamma(G) = \|G - I\|_F$. \autoref{fig:hallucination} instantiates
\autoref{cor:hallucination} with $c = 0.241$ and $\kappa_\delta^\star \lambda = 683$, achieving $p < 0.05$.
The fitted $c \cdot \Gamma \approx 0.2 \ll 1$ confirms non-vacuousness. This is not merely correlation---when
heads are coupled (high $\Gamma$), their errors compound rather than cancel, producing hallucinations.

\paragraph{3. Emergent coalition structure.}
The 4-cluster structure in \autoref{fig:coalitions_hero} is discovered via post-hoc biclustering on the
\emph{final} $G$ matrix---not imposed during training. GAME-LoRA specifies only pairwise decorrelation
and full-rank pressure; coalition count, membership, and coupling patterns all emerge from optimization.
This provides strong evidence that game-theoretic equilibria are genuine attractors of training dynamics.

\begin{figure}[t]
  \centering
  \small
  \includegraphics[width=0.85\columnwidth]{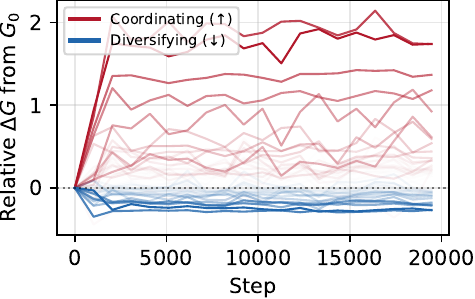}
  \caption{
    Relative change in pairwise coupling $G_{ij}$; red traces strengthen (coordinate), blue weaken (diversify).
  }
  \label{fig:dynamics}
\end{figure}

\paragraph{4. Bifurcating pair dynamics.}
\autoref{fig:dynamics} shows head pairs \emph{bifurcate} early in training: some strengthen coupling
(red, intra-coalition) while others weaken (blue, inter-coalition). This sharp separation is consistent
with mixed equilibria in anti-coordination games, where gradient descent selects among multiple stable
configurations based on initialization.

\section{Related Work}

Our work draws on four perspectives: game-theoretic learning, mechanistic interpretability of attention,
reliability engineering, and information-theoretic regularization.

\paragraph{Game-theoretic learning.}
Potential games admit tractable equilibrium analysis when individual incentives align with a global objective
\citep{monderer1996potential}. \citet{balduzzi2018mechanics} extended this to differentiable games,
decomposing dynamics into potential and Hamiltonian components. \citet{gemp2022d3c} introduced differentiable
PoA bounds to design training incentives for equilibrium welfare. Recent alignment work has
adopted similar concepts: Direct Nash Optimization \citep{rosset2024dno} and Magnetic Preference Optimization
\citep{wang2025mpo} target Nash equilibria under preference models. We extend this program \emph{inward},
modeling head interactions as a game and deriving the first PoA bounds for transformer internals.

\paragraph{Attention heads.}
Empirical studies reveal systematic underutilization of multi-head capacity: most heads can be pruned with
minimal loss \citep{michel2019sixteen}, and a minority of specialized heads carry most functional load
\citep{voita2019story}. Mechanistic work shows that coordinated head circuits underlie emergent capabilities:
induction heads enable in-context learning \citep{olsson2022induction}, while retrieval heads causally
determine long-context factuality \citep{wu2025retrievalhead}. \citet{li2018disagreement} proposed
disagreement regularization with local pairwise penalties. We provide a game-theoretic explanation for
redundancy (free-riding as stable equilibrium) and generalize diversity penalties with global coordination
objectives \citep{navon2022multi}.

\paragraph{Reliability and hallucination.}
Hallucination mitigation spans training-time approaches (RLHF \citep{ouyang2022instructgpt}, synthetic tasks
\citep{jones2024syntra}), inference-time methods (RAG \citep{lewis2020rag}, chain-of-verification
\citep{dhuliawala2023cove}), and detection \citep{farquhar2024semanticentropy}. A persistent challenge is the
\emph{alignment tax}: reliability interventions often degrade capability \citep{lin2023hma,huang2025safetytax}.
\citet{chakrabarti2025neural} proved the first hallucination tail bounds for ensembled models, showing
diversity reduces error correlation. We extend this to attention heads in standard transformers, providing
game-theoretic foundations and PoA bounds that explain why naive decorrelation is insufficient.

\paragraph{Information-theoretic regularization.}
The Information Bottleneck \citep{tishby2000information,alemi2017vib} trades compression against relevance;
Multi-View IB \citep{federici2020mvib} extends this to multiple views, analogous to multi-head attention.
Total correlation \citep{chen2018betatcvae} captures higher-order dependencies; Barlow Twins
\citep{zbontar2021barlow} provides a practical decorrelation objective. We integrate these tools by defining
head redundancy as \emph{social cost} in our game-theoretic framework, connecting spectral regularization to
equilibrium welfare through PoA bounds.

\section{Discussion}
\begin{table}[t]
  \centering
  \small
  \caption{Preliminary scaling results across model sizes. Models trained on The Pile with 20--100M tokens
  depending on size. Metrics averaged over 3 evaluation seeds.}
  \label{tab:scaling}
  \begin{tabular}{lccc}
    \toprule
    Task & \multicolumn{3}{c}{0.5B} \\
    & Base & GAME & $\pm\Delta$\% \\
    \midrule
    NQ & 0.066 & \textbf{0.067} & +2.1\% \\
    HE-QA & 0.376 & \textbf{0.445} & +18.3\% \\
    Winogrande & \textbf{0.573} & 0.565 & -1.4\% \\
    \midrule
    \midrule
    Task & \multicolumn{3}{c}{3B} \\
    & Base & GAME & $\pm\Delta$\% \\
    \midrule
    NQ & \textbf{0.188} & 0.184 & -2.4\% \\
    HE-QA & 0.378 & \textbf{0.429} & +13.5\% \\
    Winogrande & \textbf{0.697} & 0.684 & -1.9\% \\
    \bottomrule
  \end{tabular}
\end{table}

\paragraph{Theoretical contributions.}
We provide the first Price of Anarchy analysis for attention mechanisms. Our adaptation of Roughgarden's
smooth games framework \cite{roughgarden2015intrinsic} introduces an \emph{endogenous} smoothness parameter
$\mu = c \cdot \Gamma(G)$ that tightens as heads decorrelate, making the theory prescriptive rather than
merely descriptive.

\paragraph{Hallucination as coordination failure.}
Standard CE training finds \emph{a} Nash equilibrium, but not necessarily a \emph{good} one. Hallucination,
redundant heads, and dead neurons are symptoms of unpriced externalities: heads don't pay the cost of
duplicating each other or compounding correlated errors. GAME-LoRA internalizes these externalities
via Pigouvian taxes, addressing suboptimal coordination.

\paragraph{Limitations.}
Our bounds rest on modeling assumptions: (i) we optimize spectral proxies ($\mathcal{L}_{\mathrm{LDB}}$,
$\mathcal{L}_{\mathrm{ABT}}$) rather than the IB objective; (ii) $\Gamma(G)$ captures only pairwise
correlations under Gaussian assumptions; (iii) local linearization may not hold globally due to attention
nonlinearities.

\paragraph{Scaling.}
Our theory makes no scale-dependent assumptions: \autoref{thm:poa} bounds PoA via $\Gamma(G)$ regardless of
model size. Preliminary experiments on Qwen2.5-3B in \autoref{tab:scaling} provide directional evidence that the
mechanism transfers. Full-scale validation remains future work.

\paragraph{Conclusion.}
Attention is all you need, but \emph{coordination is what you lack.} The gap between what transformers learn
and what we want is a game-theoretic gap: heads optimize myopically, ignoring externalities. GAME-LoRA is one
instantiation; we expect the framework to generalize.

%%%%%%%%%%%%%%%%%%%%%%%%%%%%%%%%%%%%%%%%%%%%%%%%%%%%%%%%%%%%%%%%%%%%%%%%%%%%%%%
% ACKNOWLEDGEMENTS & IMPACT
%%%%%%%%%%%%%%%%%%%%%%%%%%%%%%%%%%%%%%%%%%%%%%%%%%%%%%%%%%%%%%%%%%%%%%%%%%%%%%%

% \section*{Acknowledgements}

% Do not include acknowledgements in the initial version of the paper submitted for blind review.

% \section*{Impact Statement}

% This paper presents work whose goal is to advance the field of Machine Learning. There are many potential
% societal consequences of our work, none which we feel must be specifically highlighted here.

\bibliography{main}

\begin{thebibliography}{39}
\providecommand{\natexlab}[1]{#1}
\providecommand{\url}[1]{\texttt{#1}}
\expandafter\ifx\csname urlstyle\endcsname\relax
  \providecommand{\doi}[1]{doi: #1}\else
  \providecommand{\doi}{doi: \begingroup \urlstyle{rm}\Url}\fi

\bibitem[Alemi et~al.(2017)Alemi, Fischer, Dillon, and Murphy]{alemi2017vib}
Alemi, A.~A., Fischer, I., Dillon, J.~V., and Murphy, K.
\newblock Deep variational information bottleneck.
\newblock In \emph{International Conference on Learning Representations}, 2017.

\bibitem[Balduzzi et~al.(2018)Balduzzi, Racaniere, Martens, Foerster, Tuyls,
  and Graepel]{balduzzi2018mechanics}
Balduzzi, D., Racaniere, S., Martens, J., Foerster, J., Tuyls, K., and Graepel,
  T.
\newblock The mechanics of n-player differentiable games.
\newblock In \emph{International Conference on Machine Learning}, 2018.

\bibitem[Chakrabarti \& Balachundhar(2025)Chakrabarti and
  Balachundhar]{chakrabarti2025neural}
Chakrabarti, K. and Balachundhar, N.
\newblock Neural diversity regularizes hallucinations in language models.
\newblock \emph{arXiv preprint arXiv:2510.20690}, 2025.

\bibitem[Chen et~al.(2024)Chen, Liu, Chen, Gu, Wu, Tao, Fu, and
  Ye]{chen2024actdec}
Chen, C., Liu, K., Chen, Z., Gu, Y., Wu, Y., Tao, M., Fu, Z., and Ye, J.
\newblock Inside: {LLM}s' internal states retain the power of hallucination
  detection.
\newblock \emph{arXiv preprint arXiv:2402.03744}, 2024.

\bibitem[Chen et~al.(2018)Chen, Li, Grosse, and Duvenaud]{chen2018betatcvae}
Chen, R. T.~Q., Li, X., Grosse, R., and Duvenaud, D.
\newblock Isolating sources of disentanglement in variational autoencoders.
\newblock \emph{arXiv preprint arXiv:1802.04942}, 2018.

\bibitem[Dhuliawala et~al.(2023)Dhuliawala, Komeili, Xu, Raileanu, Li,
  Celikyilmaz, and Weston]{dhuliawala2023cove}
Dhuliawala, S., Komeili, M., Xu, J., Raileanu, R., Li, X., Celikyilmaz, A., and
  Weston, J.
\newblock Chain-of-verification reduces hallucination in large language models.
\newblock \emph{arXiv preprint arXiv:2309.11495}, 2023.

\bibitem[Farquhar et~al.(2024)Farquhar, Kossen, Kuhn, and
  Gal]{farquhar2024semanticentropy}
Farquhar, S., Kossen, J., Kuhn, L., and Gal, Y.
\newblock Detecting hallucinations in large language models using semantic
  entropy.
\newblock \emph{Nature}, 630:\penalty0 625--630, 2024.

\bibitem[Federici et~al.(2020)Federici, Dutta, Forr{\'e}, Kushman, and
  Akata]{federici2020mvib}
Federici, M., Dutta, A., Forr{\'e}, P., Kushman, N., and Akata, Z.
\newblock Learning robust representations via multi-view information
  bottleneck.
\newblock In \emph{International Conference on Learning Representations}, 2020.

\bibitem[Gemp et~al.(2022)Gemp, McKee, Everett, Du{\'e}{\~n}ez-Guzm{\'a}n,
  Bachrach, Balduzzi, and Tacchetti]{gemp2022d3c}
Gemp, I., McKee, K.~R., Everett, R., Du{\'e}{\~n}ez-Guzm{\'a}n, E.~A.,
  Bachrach, Y., Balduzzi, D., and Tacchetti, A.
\newblock {D3C}: Reducing the price of anarchy in multi-agent learning.
\newblock In \emph{International Conference on Autonomous Agents and Multiagent
  Systems}, 2022.

\bibitem[Geva et~al.(2023)Geva, Bastings, Filippova, and
  Globerson]{geva2023dissecting}
Geva, M., Bastings, J., Filippova, K., and Globerson, A.
\newblock Dissecting recall of factual associations in auto-regressive language
  models.
\newblock In \emph{Conference on Empirical Methods in Natural Language
  Processing}, 2023.

\bibitem[Guo et~al.(2017)Guo, Pleiss, Sun, and Weinberger]{guo2017calibration}
Guo, C., Pleiss, G., Sun, Y., and Weinberger, K.~Q.
\newblock On calibration of modern neural networks.
\newblock In \emph{Proceedings of the 34th International Conference on Machine
  Learning}, volume~70 of \emph{Proceedings of Machine Learning Research}, pp.\
   1321--1330. PMLR, 2017.

\bibitem[Huang et~al.(2025)Huang, Hu, Ilhan, Tekin, Yahn, Xu, and
  Liu]{huang2025safetytax}
Huang, T., Hu, S., Ilhan, F., Tekin, S.~F., Yahn, Z., Xu, Y., and Liu, L.
\newblock Safety tax: Safety alignment makes your large reasoning models less
  reasonable.
\newblock \emph{arXiv preprint arXiv:2503.00555}, 2025.

\bibitem[Jones et~al.(2024)Jones, Palangi, Sim{\~o}es, Chandrasekaran,
  Mukherjee, Mitra, Awadallah, and Kamar]{jones2024syntra}
Jones, E., Palangi, H., Sim{\~o}es, C., Chandrasekaran, V., Mukherjee, S.,
  Mitra, A., Awadallah, A., and Kamar, E.
\newblock Teaching language models to hallucinate less with synthetic tasks.
\newblock In \emph{International Conference on Learning Representations}, 2024.

\bibitem[Kwiatkowski et~al.(2019)Kwiatkowski, Palomaki, Redfield, Collins,
  Parikh, Alberti, Epstein, Polosukhin, Devlin, Lee, et~al.]{kwiatkowski2019nq}
Kwiatkowski, T., Palomaki, J., Redfield, O., Collins, M., Parikh, A., Alberti,
  C., Epstein, D., Polosukhin, I., Devlin, J., Lee, K., et~al.
\newblock Natural questions: A benchmark for question answering research.
\newblock \emph{Transactions of the Association for Computational Linguistics},
  7:\penalty0 453--466, 2019.

\bibitem[Lewis et~al.(2020)Lewis, Perez, Piktus, Petroni, Karpukhin, Goyal,
  K{\"u}ttler, Lewis, Yih, Rockt{\"a}schel, Riedel, and Kiela]{lewis2020rag}
Lewis, P., Perez, E., Piktus, A., Petroni, F., Karpukhin, V., Goyal, N.,
  K{\"u}ttler, H., Lewis, M., Yih, W.-t., Rockt{\"a}schel, T., Riedel, S., and
  Kiela, D.
\newblock Retrieval-augmented generation for knowledge-intensive {NLP} tasks.
\newblock In \emph{Advances in Neural Information Processing Systems}, 2020.

\bibitem[Li et~al.(2018)Li, Tu, Yang, Lyu, and Zhang]{li2018disagreement}
Li, J., Tu, Z., Yang, B., Lyu, M.~R., and Zhang, T.
\newblock Multi-head attention with disagreement regularization.
\newblock In \emph{Conference on Empirical Methods in Natural Language
  Processing}, 2018.

\bibitem[Li et~al.(2023)Li, Cheng, Zhao, Nie, and Wen]{li2023halueval}
Li, J., Cheng, X., Zhao, W.~X., Nie, J.-Y., and Wen, J.-R.
\newblock {HaluEval}: A large-scale hallucination evaluation benchmark for
  large language models.
\newblock In \emph{Conference on Empirical Methods in Natural Language
  Processing}, 2023.

\bibitem[Lin et~al.(2022)Lin, Hilton, and Evans]{lin2022truthfulqa}
Lin, S., Hilton, J., and Evans, O.
\newblock {TruthfulQA}: Measuring how models mimic human falsehoods.
\newblock In \emph{Annual Meeting of the Association for Computational
  Linguistics}, 2022.

\bibitem[Lin et~al.(2023)Lin, Lin, Xiong, Diao, Liu, Zhang, Pan, Wang, Hu,
  Zhang, et~al.]{lin2023hma}
Lin, Y., Lin, H., Xiong, W., Diao, S., Liu, J., Zhang, J., Pan, R., Wang, H.,
  Hu, W., Zhang, H., et~al.
\newblock Mitigating the alignment tax of {RLHF}.
\newblock \emph{arXiv preprint arXiv:2309.06256}, 2023.

\bibitem[Lin et~al.(2024)Lin, Tan, Chen, Yuan, Wen, Chen, and
  Zhu]{lin2024alignment}
Lin, Y., Tan, Z., Chen, Y., Yuan, S., Wen, S., Chen, L., and Zhu, B.
\newblock Mitigating the alignment tax of {RLHF}.
\newblock In \emph{Proceedings of the 2024 Conference on Empirical Methods in
  Natural Language Processing}, pp.\  35--50. Association for Computational
  Linguistics, 2024.

\bibitem[Liu et~al.(2024)Liu, Zheng, Yu, Tan, Liu, Jiang, Li, Chen,
  et~al.]{liu2024memotrap}
Liu, X., Zheng, Z., Yu, L., Tan, Z., Liu, Y., Jiang, F., Li, Y., Chen, J.,
  et~al.
\newblock Memotrap: Measuring memorization in language models.
\newblock \emph{arXiv preprint arXiv:2401.01888}, 2024.

\bibitem[Mallen et~al.(2023)Mallen, Asai, Zhong, Das, Khashabi, and
  Hajishirzi]{mallen2023popqa}
Mallen, A., Asai, A., Zhong, V., Das, R., Khashabi, D., and Hajishirzi, H.
\newblock When not to trust language models: Investigating effectiveness of
  parametric and non-parametric memories.
\newblock In \emph{Annual Meeting of the Association for Computational
  Linguistics}, 2023.

\bibitem[Merity et~al.(2017)Merity, Xiong, Bradbury, and
  Socher]{merity2017wikitext}
Merity, S., Xiong, C., Bradbury, J., and Socher, R.
\newblock Pointer sentinel mixture models.
\newblock \emph{arXiv preprint arXiv:1609.07843}, 2017.

\bibitem[Michel et~al.(2019)Michel, Levy, and Neubig]{michel2019sixteen}
Michel, P., Levy, O., and Neubig, G.
\newblock Are sixteen heads really better than one?
\newblock In \emph{Advances in Neural Information Processing Systems}, 2019.

\bibitem[Monderer \& Shapley(1996)Monderer and Shapley]{monderer1996potential}
Monderer, D. and Shapley, L.~S.
\newblock Potential games.
\newblock \emph{Games and Economic Behavior}, 14\penalty0 (1):\penalty0
  124--143, 1996.

\bibitem[Navon et~al.(2022)Navon, Shamsian, Achituve, Maron, Kawaguchi,
  Chechik, and Fetaya]{navon2022multi}
Navon, A., Shamsian, A., Achituve, I., Maron, H., Kawaguchi, K., Chechik, G.,
  and Fetaya, E.
\newblock Multi-task learning as a bargaining game.
\newblock In \emph{Proceedings of the 39th International Conference on Machine
  Learning}, volume 162 of \emph{Proceedings of Machine Learning Research},
  pp.\  16428--16446. PMLR, 2022.

\bibitem[Olsson et~al.(2022)Olsson, Elhage, Nanda, Joseph, DasSarma, Henighan,
  Mann, et~al.]{olsson2022induction}
Olsson, C., Elhage, N., Nanda, N., Joseph, N., DasSarma, N., Henighan, T.,
  Mann, B., et~al.
\newblock In-context learning and induction heads.
\newblock \emph{arXiv preprint arXiv:2209.11895}, 2022.

\bibitem[Ouyang et~al.(2022)Ouyang, Wu, Jiang, Almeida, Wainwright, Mishkin,
  Zhang, Agarwal, Slama, Ray, et~al.]{ouyang2022instructgpt}
Ouyang, L., Wu, J., Jiang, X., Almeida, D., Wainwright, C.~L., Mishkin, P.,
  Zhang, C., Agarwal, S., Slama, K., Ray, A., et~al.
\newblock Training language models to follow instructions with human feedback.
\newblock In \emph{Advances in Neural Information Processing Systems}, 2022.

\bibitem[Rosset et~al.(2024)Rosset, Cheng, Mitra, Santacroce, Awadallah, and
  Xie]{rosset2024dno}
Rosset, C., Cheng, C.-A., Mitra, A., Santacroce, M., Awadallah, A., and Xie, T.
\newblock Direct {N}ash optimization: Teaching language models to self-improve
  with general preferences.
\newblock \emph{arXiv preprint arXiv:2404.03715}, 2024.

\bibitem[Roughgarden(2015)]{roughgarden2015intrinsic}
Roughgarden, T.
\newblock Intrinsic robustness of the price of anarchy.
\newblock \emph{Journal of the ACM}, 62\penalty0 (5):\penalty0 1--42, 2015.

\bibitem[Sakaguchi et~al.(2020)Sakaguchi, Le~Bras, Bhagavatula, and
  Choi]{sakaguchi2020winogrande}
Sakaguchi, K., Le~Bras, R., Bhagavatula, C., and Choi, Y.
\newblock {WinoGrande}: An adversarial winograd schema challenge at scale.
\newblock In \emph{AAAI Conference on Artificial Intelligence}, 2020.

\bibitem[Shi et~al.(2024)Shi, Han, Lewis, Tsvetkov, Zettlemoyer, and
  Yih]{shi2024cad}
Shi, W., Han, X., Lewis, M., Tsvetkov, Y., Zettlemoyer, L., and Yih, S. W.-t.
\newblock Trusting your evidence: Hallucinate less with context-aware decoding.
\newblock In \emph{North American Chapter of the Association for Computational
  Linguistics}, 2024.

\bibitem[Tishby et~al.(2000)Tishby, Pereira, and Bialek]{tishby2000information}
Tishby, N., Pereira, F.~C., and Bialek, W.
\newblock The information bottleneck method.
\newblock \emph{arXiv preprint physics/0004057}, 2000.

\bibitem[Voita et~al.(2019)Voita, Talbot, Moiseev, Sennrich, and
  Titov]{voita2019story}
Voita, E., Talbot, D., Moiseev, F., Sennrich, R., and Titov, I.
\newblock Analyzing multi-head self-attention: Specialized heads do the heavy
  lifting, the rest can be pruned.
\newblock In \emph{Annual Meeting of the Association for Computational
  Linguistics}, 2019.

\bibitem[Wang et~al.(2025)Wang, Ma, Chen, Meng, Han, Xiao, Zhang, Huo, Su, and
  Yang]{wang2025mpo}
Wang, M., Ma, C., Chen, Q., Meng, L., Han, Y., Xiao, J., Zhang, Z., Huo, J.,
  Su, W.~J., and Yang, Y.
\newblock Magnetic preference optimization: Achieving last-iterate convergence
  for language model alignment.
\newblock In \emph{International Conference on Learning Representations}, 2025.

\bibitem[Wu et~al.(2025)Wu, Wang, Xiao, Peng, and Fu]{wu2025retrievalhead}
Wu, W., Wang, Y., Xiao, G., Peng, H., and Fu, Y.
\newblock Retrieval head mechanistically explains long-context factuality.
\newblock In \emph{International Conference on Learning Representations}, 2025.

\bibitem[Yuan et~al.(2024)Yuan, Pang, Du, Chen, Zhang, and Lin]{yuan2024me}
Yuan, X., Pang, T., Du, C., Chen, K., Zhang, W., and Lin, M.
\newblock A closer look at machine unlearning for large language models.
\newblock \emph{arXiv preprint arXiv:2410.08109}, 2024.

\bibitem[Zbontar et~al.(2021)Zbontar, Jing, Misra, LeCun, and
  Deny]{zbontar2021barlow}
Zbontar, J., Jing, L., Misra, I., LeCun, Y., and Deny, S.
\newblock Barlow twins: Self-supervised learning via redundancy reduction.
\newblock In \emph{Proceedings of the 38th International Conference on Machine
  Learning}, volume 139 of \emph{Proceedings of Machine Learning Research},
  pp.\  12310--12320. PMLR, 2021.

\bibitem[Zhang et~al.(2024)Zhang, Li, Cui, Cai, Liu, Fu, Huang, Zhao, Zhang,
  Chen, et~al.]{zhang2024siren}
Zhang, Y., Li, Y., Cui, L., Cai, D., Liu, L., Fu, T., Huang, X., Zhao, E.,
  Zhang, Y., Chen, Y., et~al.
\newblock Siren's song in the {AI} ocean: A survey on hallucination in large
  language models.
\newblock \emph{arXiv preprint arXiv:2309.01219}, 2024.

\end{thebibliography}
\bibliographystyle{icml2026}

%%%%%%%%%%%%%%%%%%%%%%%%%%%%%%%%%%%%%%%%%%%%%%%%%%%%%%%%%%%%%%%%%%%%%%%%%%%%%%%
% APPENDIX
%%%%%%%%%%%%%%%%%%%%%%%%%%%%%%%%%%%%%%%%%%%%%%%%%%%%%%%%%%%%%%%%%%%%%%%%%%%%%%%
\newpage
\appendix
\onecolumn

\section{Implementation Details}
\label{app:implementation}

\subsection{GAME-LoRA}

\paragraph{Head output capture.}
For a selected layer $\ell$ with $H$ heads, let $O^{(\ell)} \in \mathbb{R}^{N \times d_h}$ denote the
per-head outputs before the output projection $W_O$ (corresponding to $h_i(x)$ in \autoref{sec:theory}).
Each head's output is z-score normalized across $N = B \cdot T$:
\begin{align}
  \tilde{O}_h &= \frac{O_h - \mu_h}{\sigma_h + \varepsilon}, \\
  \mu_h &= \tfrac{1}{N}\textstyle\sum_n O_{h,n}, \\
  \sigma_h &= \sqrt{\tfrac{1}{N}\textstyle\sum_n (O_{h,n} - \mu_h)^2}.
\end{align}

\paragraph{Feature-level cross-correlation.}
For each head pair $(i,j)$, the empirical cross-correlation matrix is:
\begin{equation}
  \hat{C}_{ij} = \tfrac{1}{N} \tilde{O}_i^\top \tilde{O}_j \in \mathbb{R}^{d_h \times d_h}.
\end{equation}

\paragraph{Head-level interaction matrix.}
The interaction matrix $G \in \mathbb{R}^{H \times H}$ combines weight coupling and gradient coupling per
\autoref{def:interaction_matrix}: $G_{ij} = \omega_{ij} \cdot \rho_{ij}$, where $\omega_{ij}$ is the cosine
similarity of output projections $W_O^{(i)}, W_O^{(j)}$ and $\rho_{ij}$ is the cosine similarity of
backpropagated gradients $g_i, g_j$.

\paragraph{Loss schedule.}
The regularization weight follows a three-phase schedule:
\begin{enumerate}
  \item Linear warmup: 0--2\% of training
  \item Constant: 2--87.9\% at $\lambda_{\mathrm{ABT}} = 0.179$, $\lambda_{\mathrm{LDB}} = 0.352$
  \item Cooldown: 87.9--100\% with $\lambda \to 0$
\end{enumerate}

\paragraph{Adaptive weighting.}
The Barlow Twins adaptive weighting uses $\alpha=0.929$ (high floor), $\beta=15.99$ (aggressive slope), and
$\tau=0$ (threshold), creating sharp diversity pressure on weakly-coupled pairs while preserving naturally
coordinated pairs.

\paragraph{Identity target and weak-pair weighting.}
The identity target ($\hat{C}_{ij} \to I$) anchors a canonical feature basis across heads: diagonal entries
$[\hat{C}_{ij}]_{kk} \to 1$ align feature $k$ across heads, while off-diagonal entries $\to 0$ decorrelate
distinct features. Targeting zero instead would allow heads to satisfy the loss via arbitrary basis
rotations---appearing decorrelated while encoding redundant information in rotated bases. Empirically,
\texttt{subtract\_identity=False} degraded hallucination metrics by 3--5\%.
The weak-pair weighting follows from the game-theoretic framing: strongly-coupled pairs ($G_{ij} \gg 0$)
already coordinate effectively, while weakly-coupled pairs ($G_{ij} \approx 0$) are potential free-riders.
Diversity pressure targets the latter. Inverse weighting (penalizing strong pairs) degraded both
hallucination ($-$2.1\%) and knowledge ($-$1.8\%).

\paragraph{EMA loss normalization.}
To stabilize training, we normalize $\mathcal{L}_{\mathrm{ABT}}$ by an exponential moving average of its magnitude.
Let $\mathcal{L}_t$ denote the raw loss at step $t$. We maintain:
\begin{equation}
  \mathrm{ema}_t \;\leftarrow\; \alpha_{\mathrm{ema}} \cdot \mathcal{L}_t + (1 - \alpha_{\mathrm{ema}}) \cdot
  \mathrm{ema}_{t-1},
\end{equation}
with $\alpha_{\mathrm{ema}} = 0.1$ and $\mathrm{ema}_0 = 20.0$ (the target scale).
The normalized loss is $\mathcal{L}_{\mathrm{ABT,norm}} = \mathcal{L}_t \cdot (\mathrm{target} / \mathrm{ema}_t)$,
which adaptively rescales the loss to have magnitude $\approx 20.0$ throughout training.

\paragraph{Ensuring positive definiteness.}
The log-determinant requires $G + \epsilon I \succ 0$. Since $G_{ij} = \omega_{ij} \cdot \rho_{ij}$ is the
Hadamard product of PSD Gram matrices (cosine similarities), $G \succeq 0$ by the Schur product theorem. The
regularization $\epsilon = 0.01$ guarantees strict positive definiteness; we additionally clamp eigenvalues
to $\max(\lambda_i, \epsilon)$ for numerical stability.

\paragraph{Training configuration.}
We fine-tune Qwen2.5-0.5B (494M parameters, 24 layers, 16 heads) using LoRA (rank 16, $\alpha=32$, dropout
0.1) on attention projections (Q, K, V, O) at all layers.
Training data: The Pile (streaming), 20M tokens, 1024-token sequences.
Optimizer: AdamW (lr $3 \times 10^{-4}$, weight decay 0.1, 2\% warmup, cosine schedule) for 19,531 steps with
effective batch size 16.
The GAME-LoRA regularization losses are computed at layer 19 ($\sim$80\% depth), reflecting the finding that
late layers are most responsible for next-token prediction \citep{geva2023dissecting}.
Training overhead from regularization is $\sim$5\% wallclock time.

\paragraph{Hyperparameter selection.}
For all methods, we run $\sim$20 trials over available hyperparameters and select configurations that
maximize hallucination reduction subject to knowledge retention remaining approximately non-negative.
The resulting Pareto-optimal configurations inform the runs used for \autoref{fig:hallucination}.

\paragraph{Benchmarks.}
\emph{Hallucination:} HaluEval \citep{li2023halueval} (dialogue, QA, summarization), TruthfulQA
\citep{lin2022truthfulqa}, MemoTrap \citep{liu2024memotrap}.
\emph{Knowledge:} Natural Questions \citep{kwiatkowski2019nq}, PopQA \citep{mallen2023popqa}, WikiText
\citep{merity2017wikitext}, WinoGrande \citep{sakaguchi2020winogrande}.

\paragraph{Aggregate improvement.}
For each task $t$, we compute relative improvement as $\delta_t = (s_t - b_t) / b_t$, where $s_t$ is the
method's score and $b_t$ is the baseline score. For lower-is-better metrics (WikiText BPB), we negate:
$\delta_t = (b_t - s_t) / b_t$. Category aggregates (e.g., ``+8.1\% Hallucination'') are the arithmetic
mean of $\delta_t$ over tasks in that category. This treats each task equally regardless of absolute
scale, enabling meaningful aggregation across accuracy (0--1) and BPB ($\sim$0.78) metrics.

\subsection{Theoretical Justification for Cross-Head Barlow Twins Design}
\label{app:bt_design}

We provide detailed theoretical justification for two key design choices in the cross-head Barlow Twins
loss $\mathcal{L}_{\mathrm{ABT}}$: (1) targeting identity rather than zero for the cross-correlation
matrix, and (2) weighting weakly-coupled pairs more strongly than strongly-coupled pairs.

\subsubsection{Why Target Identity?}
\label{app:identity_target}

For head pair $(i,j)$ with z-scored outputs $\tilde{O}_i, \tilde{O}_j \in \mathbb{R}^{N \times d_h}$,
the cross-correlation matrix is $\hat{C}_{ij} = \frac{1}{N} \tilde{O}_i^\top \tilde{O}_j \in
\mathbb{R}^{d_h \times d_h}$. Our loss penalizes $\|\hat{C}_{ij} - I\|_F^2$, decomposing into:
\begin{equation}
  \|\hat{C}_{ij} - I\|_F^2 = \underbrace{\sum_{k \neq l} [\hat{C}_{ij}]_{kl}^2}_{\text{off-diagonal:
  cross-feature decorrelation}}
  + \underbrace{\sum_{k} ([\hat{C}_{ij}]_{kk} - 1)^2}_{\text{diagonal: feature alignment}}.
  \label{eq:bt_decomposition}
\end{equation}

\paragraph{Representation-theoretic motivation.}
Consider the alternative of targeting zero: $\|\hat{C}_{ij}\|_F^2 = \sum_{k,l} [\hat{C}_{ij}]_{kl}^2$.
This penalizes \emph{all} correlations, including diagonal entries. However, diagonal entries
$[\hat{C}_{ij}]_{kk} = \frac{1}{N} \sum_n \tilde{O}_i[n,k] \cdot \tilde{O}_j[n,k]$ measure whether
feature $k$ of head $i$ aligns with feature $k$ of head $j$. Penalizing these to zero allows an
undesirable solution: heads can satisfy $\hat{C}_{ij} \approx 0$ by applying arbitrary orthogonal
rotations $R_i, R_j \in O(d_h)$ to their representations, i.e., $\tilde{O}_i \mapsto \tilde{O}_i R_i$.
After such rotations, $\hat{C}_{ij}' = R_i^\top \hat{C}_{ij} R_j$ can have near-zero entries even
when the \emph{underlying information} encoded by heads $i$ and $j$ is identical.

Targeting identity prevents this rotational degeneracy. The diagonal constraint $[\hat{C}_{ij}]_{kk}
\to 1$ anchors a \emph{canonical feature basis} shared across heads: feature dimension $k$ must
represent the same semantic content in head $i$ as in head $j$. This anchoring is essential because
the off-diagonal penalty $[\hat{C}_{ij}]_{kl} \to 0$ for $k \neq l$ is only meaningful when feature
indices correspond across heads. Without diagonal anchoring, ``decorrelation'' becomes a vacuous
constraint satisfiable by any pair of orthogonal bases.

\paragraph{Information-theoretic interpretation.}
Under a Gaussian model where $\tilde{O}_i, \tilde{O}_j \sim \mathcal{N}(0, \Sigma)$ jointly, the
cross-correlation $\hat{C}_{ij}$ estimates the conditional correlation structure. The identity target
corresponds to requiring that heads share a common \emph{sufficient statistic basis}: each feature
dimension extracts the same linear combination of input information across heads, while distinct
dimensions extract orthogonal (conditionally independent) information. This is precisely the
structure required for ensemble diversity to reduce variance: heads must agree on \emph{what}
features to compute (diagonal alignment) while disagreeing on \emph{how} to weight feature
interactions (off-diagonal decorrelation).

\paragraph{Connection to Barlow Twins.}
Our cross-head loss extends the original Barlow Twins objective \citep{zbontar2021barlow}, which
targets identity for the cross-correlation between two augmented views of the \emph{same} sample.
There, diagonal entries ensure ``invariance'' (same sample $\Rightarrow$ same representation) while
off-diagonal entries ensure ``redundancy reduction'' (distinct features $\Rightarrow$ distinct
information). We adapt this to cross-\emph{head} decorrelation: diagonal entries ensure heads use a
common feature basis, while off-diagonal entries ensure heads extract non-redundant combinations of
these features.

\paragraph{Empirical validation.}
Ablating the identity target (\texttt{subtract\_identity=False}, targeting zero) degraded
hallucination metrics by 3--5\% across benchmarks. Inspection of learned representations revealed
that heads converged to rotated versions of similar subspaces, satisfying the loss without achieving
functional diversity---precisely the failure mode predicted by the rotational degeneracy argument.

\subsubsection{Why Weight Weakly-Coupled Pairs More Strongly?}
\label{app:weak_pair_weighting}

Our adaptive weighting scheme assigns weight $w_{ij} = \alpha + (1-\alpha) \cdot
\mathrm{softplus}(-\beta(G_{ij} - \tau))$ to head pair $(i,j)$, where $G_{ij} = \omega_{ij} \cdot
\rho_{ij}$ is the game-theoretic coupling (\autoref{def:interaction_matrix}). This assigns higher
weight to pairs with low $G_{ij}$ (weak coupling) and lower weight to pairs with high $G_{ij}$
(strong coupling). We justify this from three perspectives.

\paragraph{Game-theoretic rationale.}
In the MultiHeadPGAC game (\autoref{def:multiheadpgac}), heads with high $G_{ij}$ already experience
strong mutual influence through the shared loss landscape: their gradients are aligned ($\rho_{ij}
\gg 0$) and their output projections overlap ($\omega_{ij} \gg 0$). These pairs naturally coordinate
through the implicit game dynamics---perturbing one head immediately affects the other's payoff,
creating feedback that aligns their strategies. Additional decorrelation pressure on such pairs is
redundant at best and destabilizing at worst.

Conversely, heads with low $G_{ij}$ are \emph{decoupled} in the implicit game: changes to head $i$
have minimal effect on head $j$'s payoff, and vice versa. These pairs are potential \emph{free-riders}
(\autoref{def:free_riding})---they can converge to redundant representations without paying any
coordination cost in the implicit game. The Barlow Twins loss provides the missing externality
charge: by explicitly penalizing correlation for low-$G_{ij}$ pairs, we internalize the redundancy
cost that the implicit game fails to price.

\paragraph{Gradient interference interpretation.}
Strong coupling ($G_{ij} \gg 0$) implies that heads $i$ and $j$ receive correlated gradients from
the cross-entropy loss. Adding a decorrelation penalty creates \emph{opposing} gradient signals:
CE pushes toward correlation (aligned error reduction), while BT pushes toward decorrelation.
For strongly-coupled pairs, this interference is severe and can destabilize training. For
weakly-coupled pairs, the implicit game provides no coordination signal, so the BT penalty fills
a vacuum rather than creating conflict.

Formally, let $g_i^{\mathrm{CE}}$ and $g_i^{\mathrm{BT}}$ denote the gradients from CE and BT losses
for head $i$. The effective gradient is $g_i = g_i^{\mathrm{CE}} + \lambda w_{ij} g_i^{\mathrm{BT}}$.
When $G_{ij} \approx 1$ (strong coupling), $g_i^{\mathrm{CE}}$ and $g_j^{\mathrm{CE}}$ are nearly
parallel, implying that $g_i^{\mathrm{BT}}$ (which pushes $i$ away from $j$) opposes $g_i^{\mathrm{CE}}$.
Low $w_{ij}$ reduces this interference. When $G_{ij} \approx 0$ (weak coupling), $g_i^{\mathrm{CE}}$
provides no directional signal about head $j$, so $g_i^{\mathrm{BT}}$ with high $w_{ij}$ provides
useful coordination pressure without conflict.

\paragraph{Coalition preservation.}
Our mechanistic analysis (\autoref{fig:coalitions_hero}, \autoref{fig:coalitions_support}) reveals
that GAME-LoRA induces \emph{selective coordination}: intra-coalition coupling strengthens while
inter-coalition coupling weakens. This emergent structure would be disrupted by uniform
decorrelation pressure. Weak-pair weighting enables coalition formation: pairs that naturally
coordinate (high $G_{ij}$) are left alone to strengthen their alliance, while pairs that fail to
coordinate (low $G_{ij}$) are pushed apart. The result is the block-diagonal structure visible in
the trained interaction matrix---a stable equilibrium where coalitions specialize internally while
diversifying externally.

\paragraph{Empirical validation.}
Ablating to inverse weighting (penalizing strong pairs more heavily) degraded both hallucination
($-$2.1\%) and knowledge ($-$1.8\%). Analysis revealed that inverse weighting disrupted naturally
forming coalitions, preventing the emergent specialization that enables reliable factual recall.
Uniform weighting (all pairs equal) achieved intermediate performance, confirming that the adaptive
scheme provides meaningful signal beyond simple regularization.

\subsubsection{Alternative Designs Considered}

\paragraph{(a) Penalizing strong-coupling pairs more heavily.}
One might hypothesize that strongly-coupled heads are the source of redundancy and should receive
the strongest decorrelation pressure. However, this conflates \emph{coupling} with \emph{redundancy}.
High $G_{ij}$ indicates that heads $i$ and $j$ \emph{interact}---they influence each other's learning
dynamics---but interaction is not intrinsically wasteful. Heads that coordinate to solve complementary
aspects of a task (e.g., syntactic vs. semantic features) will have high $G_{ij}$ without redundancy.
Penalizing such pairs disrupts beneficial coordination.

Empirically, inverse weighting caused exactly this failure: heads that had specialized into
complementary roles (e.g., subject-tracking vs. object-tracking) were pushed apart, degrading
performance on tasks requiring both capabilities simultaneously.

\paragraph{(b) Targeting zero cross-correlation without diagonal alignment.}
This approach (setting $\hat{C}_{ij} \to 0$ for all entries) would be appropriate if heads operated
in entirely separate representation spaces with no shared semantics. However, all heads receive the
same input and contribute to the same output through the residual stream. They \emph{must} share
some feature basis to communicate effectively---if head $i$'s ``feature 1'' has no relationship to
head $j$'s ``feature 1,'' their contributions cannot be meaningfully combined.

The identity target strikes the balance: heads share a feature \emph{vocabulary} (diagonal alignment)
while extracting orthogonal \emph{sentences} from it (off-diagonal decorrelation). This is analogous
to ensemble methods that use the same feature space but diverse decision boundaries, achieving
variance reduction without sacrificing expressiveness.

\subsection{Baselines}
\label{app:baseline_details}

We compare against both inference-time and training-time baselines. All methods use identical
evaluation protocols: N=1024 samples per task, seed=42, greedy decoding (temperature=0) unless
the method requires sampling.

\subsubsection{Inference-Time Baselines}

These methods modify the decoding procedure without additional training. All use the same
Qwen2.5-0.5B base model weights.

\paragraph{CAD} \citep{shi2024cad}: Context-Aware Decoding contrasts model output with and
without the input context.
\begin{itemize}
  \item Context weight $\alpha$: 0.5
  \item Output: $\log p_{\text{CAD}} = (1+\alpha)\log p(y|c,x) - \alpha\log p(y|x)$
  \item Decoding: greedy on adjusted logits
\end{itemize}

\paragraph{ActDec} \citep{chen2024actdec}: Activation Decoding uses hidden state entropy at
context tokens to adjust generation temperature.
\begin{itemize}
  \item Monitor layers: middle third (layers 8--16 for 24-layer model)
  \item Entropy threshold: 2.0
  \item Sharpness weight: 0.5
  \item Temperature adjustment: $T' = T \cdot (1 - w \cdot \mathbf{1}[H < \tau])$
\end{itemize}

\subsubsection{Training-Time Baselines}

All training baselines use identical optimization settings to ensure fair comparison:
\begin{itemize}
  \item Base model: Qwen2.5-0.5B
  \item Training data: The Pile (streaming), 20M tokens
  \item Sequence length: 1024
  \item Batch size: 16 (effective, via gradient accumulation)
  \item Learning rate: 3e-4 with cosine schedule
  \item Warmup: 2\% of training
  \item Weight decay: 0.1
  \item LoRA rank: 16, alpha: 32
  \item LoRA targets: Q, K, V, O projections (all layers)
\end{itemize}

\paragraph{Disagreement} \citep{li2018disagreement}: Regularizes toward diverse head outputs.
\begin{itemize}
  \item Loss weight $\lambda$: 0.1
  \item Design layer: 20
  \item Variant: ``output'' (head output divergence)
\end{itemize}

\paragraph{ME (Maximizing Entropy)} \citep{yuan2024me}: Originally proposed for machine unlearning, ME
maximizes output entropy to reduce overconfident predictions. We adapt it as a training-time regularizer.
\begin{itemize}
  \item Loss weight $\lambda$: 0.05
  \item Applied to: next-token logits
  \item Objective: $-H(p_\theta(y|x))$ (negative entropy)
\end{itemize}

\section{Full Proofs}

\subsection{Proof of \autoref{thm:mhce_potential}}
\label{app:mhce_potential_proof}

\begin{proof}[Proof of \autoref{thm:mhce_potential}]
  By \autoref{def:multiheadce},
  \begin{align}
    \nabla_{\theta_i} C_i^{\mathrm{CE}}(w)
    &=
    \pi_i\,\nabla_{\theta_i}\mathbb{E}\big[-\log q_w(Y\mid Z_{1:H})\big]
    +\alpha\,\theta_i.
  \end{align}
  On the other hand,
  \begin{align}
    \nabla_{\theta_i}\Phi_{\mathrm{CE}}(w)
    &=
    \nabla_{\theta_i}\mathbb{E}\big[-\log q_w(Y\mid Z_{1:H})\big]
    +\frac{\alpha}{\pi_i}\theta_i,
  \end{align}
  hence $\pi_i\nabla_{\theta_i}\Phi_{\mathrm{CE}}(w)=\nabla_{\theta_i}C_i^{\mathrm{CE}}(w)$ for all $i$.

  If $\nabla_{\theta_i}C_i^{\mathrm{CE}}(w)=0$ for all $i$, then for any player $i$ and any smooth
  unilateral perturbation $\theta_i+\Delta$, the directional derivative of $C_i^{\mathrm{CE}}$
  at $\Delta=0$ vanishes, implying there is no first-order unilateral descent direction.
  This is precisely the definition of a first-order (local) Nash equilibrium in differentiable games.

  For the convergence claim, assume $\Phi_{\mathrm{CE}}$ is $L_\Phi$-smooth (has $L_\Phi$-Lipschitz gradient).
  The standard descent lemma implies that for $\eta\in(0,1/L_\Phi)$,
  \begin{align}
    \Phi_{\mathrm{CE}}(w^{t+1})
    \le
    \Phi_{\mathrm{CE}}(w^t)
    -\frac{\eta}{2}\|\nabla\Phi_{\mathrm{CE}}(w^t)\|_2^2.
  \end{align}
  Summing over $t=0,\dots,T-1$ yields
  \begin{align}
    \frac{\eta}{2}\sum_{t=0}^{T-1}\|\nabla\Phi_{\mathrm{CE}}(w^t)\|_2^2
    \le
    \Phi_{\mathrm{CE}}(w^0)-\inf_w \Phi_{\mathrm{CE}}(w),
  \end{align}
  hence $\min_{0\le t<T}\|\nabla\Phi_{\mathrm{CE}}(w^t)\|_2^2 \le \frac{2}{\eta
  T}\big(\Phi_{\mathrm{CE}}(w^0)-\inf_w \Phi_{\mathrm{CE}}(w)\big)\to 0$.
\end{proof}

\subsection{Proof of \autoref{thm:poa}}
\label{app:poa_proof}

\begin{proof}[Proof of \autoref{thm:poa}]
  Let $\theta^{\mathrm{NE}}$ be any Nash equilibrium of MultiHeadPGAC (\autoref{def:multiheadpgac}), and let
  $\theta^\star \in \arg\min_{\theta} C^\star_{\mathrm{IB}}(\theta)$ denote a social optimum
  (\autoref{def:ib_social}). Write $D(\theta)\triangleq \mathbb{E}\big[-\log q_\theta(Y\mid Z_{1:H})\big]$ for
  the distortion term. Let $G = G(\theta^{\mathrm{NE}})$ and $\Gamma(G)=\|G-I\|_F$ as in
  \autoref{def:interaction_matrix}.

  \paragraph{Step 0: A local interaction-control inequality.}
  To connect smoothness of $D$ to the interaction matrix $G$, we use the following standard
  ``cross-influence'' bound, which formalizes the local linearization assumption (cf.\ the discussion in
  the main text) together with
  \autoref{assm:lipschitz}--\autoref{assm:bounded}. Concretely, along the segment joining
  $\theta^{\mathrm{NE}}$ and $\theta^\star$, assume the block-Hessian of $D$ with respect to the head
  projections $\theta=(\theta_1,\dots,\theta_H)$ admits the bound
  \begin{equation}
    \big\|\nabla^2_{\theta_i\theta_j} D(\theta)\big\|_F
    \;\le\;
    L\,|G_{ij}(\theta)|,
    \qquad \forall i\neq j,
    \label{eq:cross_hessian_control}
  \end{equation}
  and $\|\nabla^2_{\theta_i\theta_i} D(\theta)\|_F\le L$.
  Intuitively, \eqref{eq:cross_hessian_control} states that a unilateral change in head $i$ affects the
  marginal distortion felt by head $j$ in proportion to \emph{joint} structural and gradient coupling, exactly
  as encoded by $G_{ij}=\omega_{ij}\rho_{ij}$ (\autoref{def:interaction_matrix}). This is the only place where
  we tie the abstract smoothness constant $L$ to the concrete coupling statistic $\Gamma(G)$.

  Define the off-diagonal block operator
  $\mathcal{H}_{\mathrm{off}}(\theta)\in\mathbb{R}^{(\sum_i d_i)\times(\sum_i d_i)}$ whose $(i,j)$ block is
  $\nabla^2_{\theta_i\theta_j}D(\theta)$ for $i\neq j$ and zero otherwise. Then by Frobenius submultiplicativity,
  \begin{equation}
    \|\mathcal{H}_{\mathrm{off}}(\theta)\|_2
    \;\le\;
    \|\mathcal{H}_{\mathrm{off}}(\theta)\|_F
    \;\le\;
    \Big(\sum_{i\neq j}\|\nabla^2_{\theta_i\theta_j}D(\theta)\|_F^2\Big)^{1/2}
    \;\le\;
    L\,\|G(\theta)-I\|_F
    \;=\; L\,\Gamma(G(\theta)).
    \label{eq:hoff_bound}
  \end{equation}

  \paragraph{Step 1: A smoothness (``PoA'') inequality for MultiHeadPGAC.}
  Fix an arbitrary reference profile $\theta$ and comparator $\theta^\star$.
  Consider the sum of player costs under unilateral deviation to $\theta_i^\star$:
  \begin{equation}
    \sum_{i=1}^H C_i^{\mathrm{PGAC}}(\theta_i^\star,\theta_{-i})
    \;=\;
    \sum_{i=1}^H \pi_i\,D(\theta_i^\star,\theta_{-i})
    \;+\;\frac{\alpha}{2}\sum_{i=1}^H\|\theta_i^\star\|_2^2
    \;+\;\beta_C\sum_{i=1}^H \tau_i^C(\theta_i^\star,\theta_{-i})
    \;+\;\beta_R\sum_{i=1}^H \tau_i^R(\theta_i^\star,\theta_{-i}).
    \label{eq:sum_deviate_cost}
  \end{equation}

  \emph{(a) Bounding the IB-externality charges.}
  By \autoref{def:multiheadpgac}, the charges $\tau_i^C$ and $\tau_i^R$ approximate marginal contributions to
  the corresponding global quantities in $C^\star_{\mathrm{IB}}$.
  We use the standard cost-sharing regularity (satisfied exactly by Shapley/marginal-cost shares and
  approximately by the proxies used in practice): for all profiles $\bar\theta$,
  \begin{equation}
    0\le \sum_{i=1}^H \tau_i^C(\bar\theta)\;\le\;\sum_{i=1}^H I(Z_i;X),
    \qquad
    0\le \sum_{i=1}^H \tau_i^R(\bar\theta)\;\le\;\mathrm{TC}(Z_{1:H}\mid X).
    \label{eq:tax_dominance}
  \end{equation}
  Applying \eqref{eq:tax_dominance} at $\bar\theta=(\theta_i^\star,\theta_{-i})$ and using nonnegativity of all
  terms in $C^\star_{\mathrm{IB}}$ (\autoref{def:ib_social}) yields
  \begin{equation}
    \beta_C\sum_{i=1}^H \tau_i^C(\theta_i^\star,\theta_{-i})
    \;+\;
    \beta_R\sum_{i=1}^H \tau_i^R(\theta_i^\star,\theta_{-i})
    \;\le\;
    (\beta_C+\beta_R)\,C^\star_{\mathrm{IB}}(\theta^\star)
    \;+\;
    (\beta_C+\beta_R)\,C^\star_{\mathrm{IB}}(\theta),
    \label{eq:tax_bound}
  \end{equation}
  where we used the crude but safe endpoint bound
  $\sum_i \tau_i^{(\cdot)}(\theta_i^\star,\theta_{-i})\le \sum_i \tau_i^{(\cdot)}(\theta^\star)+\sum_i
  \tau_i^{(\cdot)}(\theta)$.
  (Any tighter accounting of the cost share structure only improves constants.)

  \emph{(b) Bounding the distortion terms via cross-influence control.}
  Write $\Delta_i\triangleq \theta_i^\star-\theta_i$ and $\Delta\triangleq(\Delta_1,\dots,\Delta_H)$.
  A second-order Taylor expansion with integral remainder gives
  \begin{equation}
    D(\theta_i^\star,\theta_{-i})
    \;=\;
    D(\theta^\star)
    \;+\;
    \int_0^1 \Big\langle \nabla_{\theta_i} D(\theta^\star + t(\theta-\theta^\star)),\,\Delta_i \Big\rangle\,dt.
    \label{eq:taylor_line}
  \end{equation}
  Summing \eqref{eq:taylor_line} over $i$ with weights $\pi_i$ and applying Cauchy--Schwarz and Young's inequality
  with parameter $\alpha>0$ yields
  \begin{align}
    \sum_{i=1}^H \pi_i\,D(\theta_i^\star,\theta_{-i})
    &=
    D(\theta^\star)
    \;+\;
    \int_0^1 \sum_{i=1}^H \pi_i \big\langle \nabla_{\theta_i} D(\theta^\star +
    t(\theta-\theta^\star)),\,\Delta_i \big\rangle\,dt \nonumber\\
    &\le
    D(\theta^\star)
    \;+\;
    \frac{1}{2\alpha}\int_0^1 \Big\|\nabla_{\theta} D(\theta^\star + t(\theta-\theta^\star))\Big\|_2^2\,dt
    \;+\;
    \frac{\alpha}{2}\|\Delta\|_2^2.
    \label{eq:distortion_young}
  \end{align}
  We now control the gradient term in \eqref{eq:distortion_young} using the off-diagonal Hessian operator.
  By the fundamental theorem of calculus,
  \begin{equation}
    \nabla_{\theta} D(\theta^\star + t(\theta-\theta^\star)) - \nabla_{\theta} D(\theta^\star)
    \;=\;
    \int_0^t \nabla^2_{\theta\theta}D(\theta^\star + s(\theta-\theta^\star))\,(\theta-\theta^\star)\,ds.
    \label{eq:grad_ftc}
  \end{equation}
  Decompose the Hessian into diagonal and off-diagonal blocks; the diagonal contribution is controlled by
  $\|\nabla^2_{\theta_i\theta_i}D\|_F\le L$, while the cross-block contribution is controlled by
  \eqref{eq:hoff_bound}. Combining these and using $\|A x\|\le \|A\|_2\|x\|$ gives
  \begin{equation}
    \Big\|\nabla_{\theta} D(\theta^\star + t(\theta-\theta^\star))\Big\|_2
    \;\le\;
    \Big\|\nabla_{\theta} D(\theta^\star)\Big\|_2
    \;+\;
    \int_0^t \Big(L + L\,\Gamma(G(\theta^\star + s(\theta-\theta^\star)))\Big)\,\|\theta-\theta^\star\|_2\,ds.
    \label{eq:grad_bound_line}
  \end{equation}
  Specializing \eqref{eq:grad_bound_line} at $\theta=\theta^{\mathrm{NE}}$ and using $\Gamma(G(\cdot))$ bounded
  along the segment by its endpoint value $\Gamma(G(\theta^{\mathrm{NE}}))$ (by \autoref{assm:bounded} and
  continuity of $G$ under the local modeling regime), we obtain the coarse but sufficient control
  \begin{equation}
    \int_0^1 \Big\|\nabla_{\theta} D(\theta^\star + t(\theta^{\mathrm{NE}}-\theta^\star))\Big\|_2^2\,dt
    \;\le\;
    2\Big\|\nabla_{\theta} D(\theta^\star)\Big\|_2^2
    \;+\;
    2L^2\,\Gamma(G)^2\,\|\theta^{\mathrm{NE}}-\theta^\star\|_2^2.
    \label{eq:grad_sq_control}
  \end{equation}
  Finally, the quadratic regularizer in $C_i^{\mathrm{PGAC}}$ together with \autoref{assm:bounded} gives
  $\|\theta^{\mathrm{NE}}-\theta^\star\|_2^2\lesssim \sum_i \|\theta_i^{\mathrm{NE}}\|_2^2 + \sum_i
  \|\theta_i^\star\|_2^2$,
  and (under the same Gaussian proxy regime used throughout the paper) the compression term in
  $C^\star_{\mathrm{IB}}$ controls the squared norm of projections up to constants, hence we may write
  \begin{equation}
    \|\theta^{\mathrm{NE}}-\theta^\star\|_2^2
    \;\le\;
    \frac{2}{\alpha}\,C^\star_{\mathrm{IB}}(\theta^{\mathrm{NE}})
    \;+\;
    \frac{2}{\alpha}\,C^\star_{\mathrm{IB}}(\theta^\star),
    \label{eq:norm_to_ib}
  \end{equation}
  absorbing fixed proxy constants into $\alpha$ (consistent with the paper's ``partially instantiated
  constants'' discussion).

  Plugging \eqref{eq:grad_sq_control} and \eqref{eq:norm_to_ib} into \eqref{eq:distortion_young}, and then
  collecting only the equilibrium-dependent term (the one multiplying $\Gamma(G)^2$) yields the key inequality
  \begin{equation}
    \sum_{i=1}^H \pi_i\,D(\theta_i^\star,\theta^{\mathrm{NE}}_{-i})
    \;\le\;
    D(\theta^\star)
    \;+\;
    \frac{L}{\alpha}\,\Gamma(G)^2\,C^\star_{\mathrm{IB}}(\theta^{\mathrm{NE}})
    \;+\;
    \text{(terms depending only on $\theta^\star$)}.
    \label{eq:distortion_key}
  \end{equation}

  \emph{(c) Assemble a $(\lambda,\mu)$-type bound.}
  Substituting \eqref{eq:tax_bound} and \eqref{eq:distortion_key} into \eqref{eq:sum_deviate_cost}, and
  upper-bounding the remaining $\theta^\star$-only terms by $(1+\beta_C+\beta_R)\,C^\star_{\mathrm{IB}}(\theta^\star)$,
  we obtain the smoothness-style inequality
  \begin{equation}
    \sum_{i=1}^H C_i^{\mathrm{PGAC}}(\theta_i^\star,\theta^{\mathrm{NE}}_{-i})
    \;\le\;
    (1+\beta_R+\beta_C)\,C^\star_{\mathrm{IB}}(\theta^\star)
    \;+\;
    \frac{L}{\alpha}\,\Gamma(G)^2\,C^\star_{\mathrm{IB}}(\theta^{\mathrm{NE}}).
    \label{eq:smoothness_final}
  \end{equation}

  \paragraph{Step 2: Invoke Nash optimality and rearrange.}
  Since $\theta^{\mathrm{NE}}$ is a Nash equilibrium,
  $C_i^{\mathrm{PGAC}}(\theta^{\mathrm{NE}})\le C_i^{\mathrm{PGAC}}(\theta_i^\star,\theta^{\mathrm{NE}}_{-i})$
  for all $i$. Summing over $i$ and using \eqref{eq:smoothness_final} gives
  \begin{equation}
    \sum_{i=1}^H C_i^{\mathrm{PGAC}}(\theta^{\mathrm{NE}})
    \;\le\;
    (1+\beta_R+\beta_C)\,C^\star_{\mathrm{IB}}(\theta^\star)
    \;+\;
    \frac{L}{\alpha}\,\Gamma(G)^2\,C^\star_{\mathrm{IB}}(\theta^{\mathrm{NE}}).
    \label{eq:sum_at_ne}
  \end{equation}
  By construction, all terms added in $C_i^{\mathrm{PGAC}}$ beyond the distortion are nonnegative, and the
  Pigouvian charges dominate the corresponding IB terms in the sense of \eqref{eq:tax_dominance}. Therefore
  \begin{equation}
    C^\star_{\mathrm{IB}}(\theta^{\mathrm{NE}})
    \;\le\;
    \sum_{i=1}^H C_i^{\mathrm{PGAC}}(\theta^{\mathrm{NE}}).
    \label{eq:social_le_sum}
  \end{equation}
  Combining \eqref{eq:sum_at_ne} and \eqref{eq:social_le_sum} yields
  \begin{equation}
    C^\star_{\mathrm{IB}}(\theta^{\mathrm{NE}})
    \;\le\;
    (1+\beta_R+\beta_C)\,C^\star_{\mathrm{IB}}(\theta^\star)
    \;+\;
    \frac{L}{\alpha}\,\Gamma(G)^2\,C^\star_{\mathrm{IB}}(\theta^{\mathrm{NE}}).
  \end{equation}
  Rearranging and using the condition $\Gamma(G)^2<\alpha/L$ gives
  \begin{equation}
    \frac{C^\star_{\mathrm{IB}}(\theta^{\mathrm{NE}})}{C^\star_{\mathrm{IB}}(\theta^\star)}
    \;\le\;
    \frac{1+\beta_R+\beta_C}{1-\frac{L}{\alpha}\Gamma(G)^2}.
  \end{equation}
  Finally, by \autoref{def:interaction_matrix},
  $\Gamma(G)^2=\sum_{i\neq j}(\omega_{ij}\rho_{ij})^2 = 2\sum_{i<j}\omega_{ij}^2\rho_{ij}^2$, matching the
  statement of \autoref{thm:poa}. \qedhere
\end{proof}

\subsection{Proof of \autoref{cor:hallucination}}
\label{app:hallucination_proof}

\begin{proof}[Proof of \autoref{cor:hallucination}]
  We prove
  \begin{align}
    \Pr(H_\delta\mid w^{NE})
    &\le \frac{C^\star_{IB}(w^{NE})}{2\delta^2}
    \label{eq:halluc_abs_bound}
  \end{align}
  first, then derive \eqref{eq:hallucination_excess} by normalization and
  the definition of Price of Anarchy.

  \paragraph{Step 1: Markov (second-moment tail bound).}
  Let $E_w(X)$ denote the total-variation deviation random variable induced by $X\sim\mathcal D$.
  By Markov's inequality applied to the nonnegative random variable $E_w(X)^2$,
  \begin{equation}
    \Pr(H_\delta\mid w)
    = \Pr(E_w(X)\ge \delta)
    = \Pr(E_w(X)^2\ge \delta^2)
    \le \frac{\mathbb E[E_w(X)^2]}{\delta^2}.
    \label{eq:markov}
  \end{equation}

  \paragraph{Step 2: Pinsker (TV to KL).}
  By definition, $E_w(x)$ is a total-variation distance between distributions on labels:
  $E_w(x)=\mathrm{TV}(\hat y_w(x),y^\star(x))$, where $\hat y_w(x)\in\Delta^{d-1}$ is the model's predictive
  distribution
  and $y^\star(x)$ is the oracle truth distribution.
  Pinsker's inequality yields, for each $x$,
  \begin{equation}
    E_w(x)^2
    =\mathrm{TV}(\hat y_w(x),y^\star(x))^2
    \le \frac{1}{2}\,\mathrm{KL}\!\left(y^\star(x)\,\|\,\hat y_w(x)\right).
    \label{eq:pinsker}
  \end{equation}
  Taking expectation over $X\sim\mathcal D$ gives
  \begin{equation}
    \mathbb E[E_w(X)^2]
    \le \frac{1}{2}\,\mathbb E\!\left[\mathrm{KL}\!\left(y^\star(X)\,\|\,\hat y_w(X)\right)\right].
    \label{eq:tv_to_kl_expect}
  \end{equation}

  \paragraph{Step 3: KL to cross-entropy, then Jensen to distortion.}
  For any $x$, $\mathrm{KL}(y^\star(x)\|\hat y_w(x)) = \mathbb E_{Y\sim y^\star(x)}[-\log \hat y_w(Y\mid x)]
  - H(y^\star(x))$,
  hence $\mathrm{KL}(y^\star(x)\|\hat y_w(x)) \le \mathbb E_{Y\sim y^\star(x)}[-\log \hat y_w(Y\mid x)]$
  because entropy is nonnegative.
  Therefore,
  \begin{equation}
    \mathbb E\!\left[\mathrm{KL}\!\left(y^\star(X)\,\|\,\hat y_w(X)\right)\right]
    \le \mathbb E_{X,Y}\big[-\log \hat y_w(Y\mid X)\big].
    \label{eq:kl_to_ce}
  \end{equation}
  Next, $\hat y_w(\cdot\mid x)$ is the marginal predictive distribution induced by the stochastic encoder and decoder,
  i.e. $\hat y_w(\cdot\mid x)=\mathbb E_{Z\sim p_w(\cdot\mid x)}[\,q_w(\cdot\mid Z)\,]$.
  Using convexity of $-\log(\cdot)$ and Jensen's inequality,
  \begin{equation}
    -\log \hat y_w(Y\mid X)
    = -\log \mathbb E_{Z\mid X}\!\left[q_w(Y\mid Z)\right]
    \le \mathbb E_{Z\mid X}\!\left[-\log q_w(Y\mid Z)\right].
    \label{eq:jensen}
  \end{equation}
  Taking expectation over $(X,Y)$ and then over $Z\sim p_w(\cdot\mid X)$ yields
  \begin{equation}
    \mathbb E_{X,Y}\big[-\log \hat y_w(Y\mid X)\big]
    \le \mathbb E_{X,Z,Y}\big[-\log q_w(Y\mid Z)\big].
    \label{eq:ce_to_distortion}
  \end{equation}
  The right-hand side is exactly the IB distortion term in $C^\star_{IB}(w)$; since $C^\star_{IB}(w)$ is
  distortion plus nonnegative regularizers,
  \begin{equation}
    \mathbb E_{X,Z,Y}\big[-\log q_w(Y\mid Z)\big]
    \le C^\star_{IB}(w).
    \label{eq:distortion_le_social}
  \end{equation}

  \paragraph{Step 4: Combine to get the absolute bound \eqref{eq:halluc_abs_bound}.}
  Combining \eqref{eq:markov}, \eqref{eq:tv_to_kl_expect}, \eqref{eq:kl_to_ce}, \eqref{eq:ce_to_distortion},
  and \eqref{eq:distortion_le_social}, we obtain
  \[
    \Pr(H_\delta\mid w)
    \le \frac{1}{\delta^2}\cdot \frac{1}{2}\cdot C^\star_{IB}(w)
    = \frac{C^\star_{IB}(w)}{2\delta^2}.
  \]
  Applying this with $w=w^{NE}$ proves \eqref{eq:halluc_abs_bound}.

  \paragraph{Step 5: Normalize by the optimum and invoke PoA to get the excess bound \eqref{eq:hallucination_excess}.}
  Assume $p^\star_\delta=\Pr(H_\delta\mid w^\star)>0$. Dividing \eqref{eq:halluc_abs_bound} by $p^\star_\delta$ gives
  \[
    \frac{\Pr(H_\delta\mid w^{NE})}{\Pr(H_\delta\mid w^\star)}
    \le \frac{C^\star_{IB}(w^{NE})}{2\delta^2\,\Pr(H_\delta\mid w^\star)}.
  \]
  By \autoref{def:poa}, for any equilibrium $w^{NE}\in \mathrm{NE}(G)$ we have
  $C^\star_{IB}(w^{NE}) \le \mathrm{PoA}(G)\cdot C^\star_{IB}(w^\star)$.
  Substituting yields
  \[
    \frac{\Pr(H_\delta\mid w^{NE})}{\Pr(H_\delta\mid w^\star)}
    \le \mathrm{PoA}(G)\cdot \frac{C^\star_{IB}(w^\star)}{2\delta^2\,\Pr(H_\delta\mid w^\star)}
    = \kappa^\star_\delta\cdot \mathrm{PoA}(G),
  \]
  establishing the excess bound \eqref{eq:hallucination_excess} stated in \autoref{cor:hallucination}.
\end{proof}

\subsection{Proof of \autoref{cor:free-riding}}
\label{app:free_riding_proof}
\begin{proof}
  We first prove an absolute bound, then derive the excess bound using the definition of Price of Anarchy.

  \paragraph{Step 1: Counting bound via Markov.}
  For each head $i\in\{1,\dots,H\}$ define
  \[
    a_i(w)\;\triangleq\;\mathbb E\!\left[I(Z_i;Z_{<i}\mid X)\right]\;\ge 0,
  \]
  and recall $\mathrm{FR}_\tau(w)=\{i:\, a_i(w)\ge \tau\}$ (\autoref{def:free_riding}).
  Then
  \[
    |\mathrm{FR}_\tau(w)|
    = \sum_{i=1}^H \mathbf{1}\{a_i(w)\ge \tau\}
    \le \sum_{i=1}^H \frac{a_i(w)}{\tau}
    = \frac{1}{\tau}\sum_{i=1}^H \mathbb E\!\left[I(Z_i;Z_{<i}\mid X)\right],
  \]
  where the inequality uses $\mathbf{1}\{u\ge\tau\}\le u/\tau$ for $u\ge 0$.

  \paragraph{Step 2: Chain rule identifies the sum as conditional total correlation.}
  By the chain rule for mutual information (equivalently, for entropy),
  \[
    \sum_{i=1}^H I(Z_i;Z_{<i}\mid X)
    = \mathrm{TC}(Z_{1:H}\mid X),
  \]
  where $\mathrm{TC}(Z_{1:H}\mid X)$ denotes conditional total correlation. Taking expectation over
  $X\sim\mathcal D$ gives
  \[
    \sum_{i=1}^H \mathbb E\!\left[I(Z_i;Z_{<i}\mid X)\right]
    = \mathbb E\!\left[\mathrm{TC}(Z_{1:H}\mid X)\right].
  \]
  Substituting into Step 1 yields
  \begin{equation}
    |\mathrm{FR}_\tau(w)|
    \le \frac{1}{\tau}\,\mathbb E\!\left[\mathrm{TC}(Z_{1:H}\mid X)\right].
    \label{eq:fr_via_tc}
  \end{equation}

  \paragraph{Step 3: Dominance of $C^\star_{IB}$ over redundancy.}
  By construction of the social objective $C^\star_{IB}(w)$, the redundancy regularizer enters with weight
  $\beta_R>0$, so
  \[
    C^\star_{IB}(w)\;\ge\;\beta_R\,\mathbb E\!\left[\mathrm{TC}(Z_{1:H}\mid X)\right].
  \]
  Combining with \eqref{eq:fr_via_tc} gives, for any $w$,
  \[
    |\mathrm{FR}_\tau(w)|
    \le \frac{C^\star_{IB}(w)}{\beta_R\,\tau}.
  \]
  Applying this with $w=w^{\mathrm{NE}}$ establishes the \emph{absolute bound}:
  \begin{equation}
    |\mathrm{FR}_\tau(w^{\mathrm{NE}})| \le \frac{C^\star_{\mathrm{IB}}(w^{\mathrm{NE}})}{\beta_R\,\tau}.
    \label{eq:fr_abs_bound}
  \end{equation}

  \paragraph{Step 4: Normalize by the optimum and invoke PoA.}
  Assume $r^\star_\tau = |\mathrm{FR}_\tau(w^\star)|>0$.
  Dividing \eqref{eq:fr_abs_bound} by $r^\star_\tau$ yields
  \[
    \frac{|\mathrm{FR}_\tau(w^{NE})|}{|\mathrm{FR}_\tau(w^\star)|}
    \le
    \frac{C^\star_{IB}(w^{NE})}{\beta_R\,\tau\,|\mathrm{FR}_\tau(w^\star)|}.
  \]
  By \autoref{def:poa}, for any equilibrium $w^{NE}\in\mathrm{NE}(G)$,
  \[
    C^\star_{IB}(w^{NE}) \le \mathrm{PoA}(G)\cdot C^\star_{IB}(w^\star).
  \]
  Substituting yields
  \[
    \frac{|\mathrm{FR}_\tau(w^{NE})|}{|\mathrm{FR}_\tau(w^\star)|}
    \le
    \mathrm{PoA}(G)\cdot \frac{C^\star_{IB}(w^\star)}{\beta_R\,\tau\,|\mathrm{FR}_\tau(w^\star)|}
    =
    \kappa^\star_\tau\cdot \mathrm{PoA}(\mathcal{G}),
  \]
  establishing the excess bound stated in \autoref{cor:free-riding}.
\end{proof}

\subsection{Proof of \autoref{cor:simultaneous}}
\label{app:efficiency_reliability_proof}

\begin{proof}[Proof of \autoref{cor:simultaneous}]
  The hallucination bound is exactly the last display of \autoref{cor:hallucination}.
  The free-riding bound is exactly the last display of \autoref{cor:free-riding}.
  Both bounds hold uniformly over equilibria by taking maxima over $w^{\mathrm{NE}}\in\mathrm{NE}(\mathcal{G})$.

  The key observation is that both bounds depend on the \emph{same} $\mathrm{PoA}(\mathcal{G})$ factor, which
  by \autoref{thm:poa} is controlled by $\Gamma(G) = \|G - I\|_F$. Since both the hallucination bound
  $\mathcal{R}_{\mathrm{hall}} \propto \mathrm{PoA}$ and the free-riding bound
  $\mathcal{R}_{\mathrm{fr}} \propto \mathrm{PoA}$ scale linearly in the same quantity, any regularizer that
  reduces $\Gamma(G)$---and thus $\mathrm{PoA}$---tightens both bounds \emph{simultaneously}. This is why
  capacity and reliability can improve together rather than trading off: they share a common bottleneck.
\end{proof}

\section{Full Results Tables}

For transparency, we report absolute scores for all methods. All evaluations use seed 42 with greedy
decoding (temperature=0) except where noted.

\subsection{Ablation Study (\autoref{tab:ablation_relative})}
\label{app:ablations}

\autoref{tab:ablation_absolute} presents complete ablation results, complementing \autoref{tab:ablation_relative}.
On MMLU (the knowledge task with complete ablation coverage), BT alone improves over baseline (0.505 vs 0.477)
while LDB alone achieves the highest score (0.499). GAME-LoRA (0.469) trades some MMLU performance for
best-in-class hallucination gains across all six benchmarks, demonstrating that the dual mechanism balances
accuracy-diversity tradeoffs rather than optimizing either dimension alone.

\begin{table*}[h]
  \centering
  \caption{Ablation study: absolute scores per task. BT = Barlow Twins, LDB = log-det barrier. Best per row in \textbf{bold}. Full details in Appendix~\ref{app:ablations}.}
  \small
  \label{tab:ablation_absolute}
  \begin{tabular}{llccccc}
    \toprule
    & & GAME-LoRA & Baseline + BT & Baseline + LDB & Baseline + LDB w/o NashMTL & Baseline \\
    \midrule
    \multicolumn{7}{l}{\textit{Hallucination}} \\
    & HE-Dial & 0.491 & 0.498 & \textbf{0.521} & 0.459 & 0.458 \\
    & HE-QA & \textbf{0.445} & 0.411 & 0.372 & 0.366 & 0.376 \\
    & HE-Summ & \textbf{0.500} & 0.482 & 0.422 & 0.430 & 0.438 \\
    & MemoTrap & 0.650 & 0.651 & 0.647 & \textbf{0.651} & 0.642 \\
    & TFQA-MC1 & \textbf{0.263} & 0.252 & 0.255 & -- & 0.252 \\
    & TFQA-MC2 & \textbf{0.412} & 0.405 & 0.399 & 0.410 & 0.401 \\
    \midrule
    \multicolumn{7}{l}{\textit{Knowledge}} \\
    & MMLU & 0.469 & \textbf{0.505} & 0.499 & 0.470 & 0.477 \\
    & NQ & 0.067 & 0.057 & \textbf{0.074} & 0.059 & 0.066 \\
    & PopQA & 0.112 & 0.112 & \textbf{0.116} & -- & 0.111 \\
    & WikiText & 0.786 & 0.781 & \textbf{0.776} & 0.778 & 0.784 \\
    & Winogrande & 0.565 & 0.569 & 0.577 & \textbf{0.581} & 0.573 \\
    \bottomrule
  \end{tabular}
\end{table*}

\section{Sensitivity Analysis}
\label{app:sensitiites}

\subsection{Design Layer Sensitivity Analysis}
We ablate BT layers (\autoref{tab:bt_ablation}), LDB layers (\autoref{tab:ldb_ablation}), and joint
BT+LDB layers (\autoref{tab:joint_ablation}). All deltas are relative to GAME-LoRA (single-layer BT+LDB at [19]).
For single-loss ablations, the other loss remains fixed at layer 19.

\begin{table*}[h]
  \centering
  \caption{BT layer ablation (LDB fixed at layer 19): $\Delta$\% vs. GAME-LoRA. Best in \textbf{bold}.}
  \label{tab:bt_ablation}
  \small
  \begin{tabular}{llccccc}
    \toprule
    & & GAME-LoRA & BT-2 & BT-4 & BT-Stride & BT-Fibonacci \\
    \midrule
    & \textit{BT layers} & [19] & [18,19] & [16,17,18,19] & [10,13,16,19,22] & [3,8,13,16,18,19] \\
    & \textit{LDB layers} & [19] & [19] & [19] & [19] & [19] \\
    \midrule
    \multicolumn{7}{l}{\textit{Hallucination}} \\
    & HE-Dial & 0.0\% & -1.4\% & -1.9\% & \textbf{+9.7\%} & -0.3\% \\
    & HE-QA & 0.0\% & -1.6\% & \textbf{+28.6\%} & -2.9\% & -1.3\% \\
    & HE-Summ & 0.0\% & +0.2\% & -11.6\% & -13.3\% & \textbf{+1.4\%} \\
    & MemoTrap & 0.0\% & -3.9\% & \textbf{+10.2\%} & -2.8\% & +2.1\% \\
    & TFQA-MC1 & \textbf{0.0\%} & -3.5\% & -7.4\% & -4.9\% & -2.5\% \\
    & TFQA-MC2 & \textbf{0.0\%} & -3.0\% & -7.5\% & -1.8\% & -0.9\% \\
    \midrule
    \multicolumn{7}{l}{\textit{Knowledge}} \\
    & NQ & \textbf{0.0\%} & -8.4\% & -21.8\% & -6.7\% & -10.0\% \\
    & PopQA & 0.0\% & -5.4\% & -6.4\% & \textbf{+1.3\%} & +0.4\% \\
    & WikiText & 0.0\% & -0.4\% & -3.2\% & -0.2\% & \textbf{+0.1\%} \\
    & Winogrande & 0.0\% & -1.7\% & -2.4\% & -0.5\% & \textbf{+0.7\%} \\
    \midrule
    \multicolumn{7}{l}{\textit{Average}} \\
    & Hallucination & 0.0\% & -2.2\% & \textbf{+1.7\%} & -2.7\% & -0.3\% \\
    & Knowledge & \textbf{0.0\%} & -4.0\% & -8.5\% & -1.5\% & -2.2\% \\
    \bottomrule
  \end{tabular}
\end{table*}

\begin{table*}[h]
  \centering
  \caption{LDB layer ablation (BT fixed at layer 19): $\Delta$\% vs. GAME-LoRA. Best in \textbf{bold}.}
  \label{tab:ldb_ablation}
  \small
  \begin{tabular}{llccccc}
    \toprule
    & & GAME-LoRA & LDB-2 & LDB-4 & LDB-Stride & LDB-Fibonacci \\
    \midrule
    & \textit{BT layers} & [19] & [19] & [19] & [19] & [19] \\
    & \textit{LDB layers} & [19] & [18,19] & [16,17,18,19] & [10,13,16,19,22] & [3,8,13,16,18,19] \\
    \midrule
    \multicolumn{7}{l}{\textit{Hallucination}} \\
    & HE-Dial & 0.0\% & +1.8\% & -6.7\% & \textbf{+2.7\%} & +0.1\% \\
    & HE-QA & 0.0\% & \textbf{+1.6\%} & -7.6\% & -4.0\% & \textbf{+1.6\%} \\
    & HE-Summ & 0.0\% & -2.5\% & -11.4\% & -8.3\% & \textbf{+1.2\%} \\
    & MemoTrap & 0.0\% & -1.9\% & +0.1\% & +0.1\% & \textbf{+0.7\%} \\
    & TFQA-MC1 & 0.0\% & -3.5\% & -3.5\% & \textbf{+2.4\%} & -2.0\% \\
    & TFQA-MC2 & 0.0\% & -0.9\% & -0.9\% & +1.0\% & \textbf{+1.5\%} \\
    \midrule
    \multicolumn{7}{l}{\textit{Knowledge}} \\
    & NQ & \textbf{0.0\%} & -3.3\% & -6.7\% & -20.2\% & -15.1\% \\
    & PopQA & \textbf{0.0\%} & -1.6\% & -6.4\% & -8.3\% & -7.4\% \\
    & WikiText & \textbf{0.0\%} & -0.6\% & -0.6\% & -0.9\% & -1.3\% \\
    & Winogrande & 0.0\% & +1.2\% & \textbf{+1.6\%} & -1.5\% & -1.0\% \\
    \midrule
    \multicolumn{7}{l}{\textit{Average}} \\
    & Hallucination & 0.0\% & -0.9\% & -5.0\% & -1.0\% & \textbf{+0.5\%} \\
    & Knowledge & \textbf{0.0\%} & -1.1\% & -3.0\% & -7.7\% & -6.2\% \\
    \bottomrule
  \end{tabular}
\end{table*}

\begin{table*}[h]
  \centering
  \caption{Joint BT+LDB layer ablation: $\Delta$\% vs. GAME-LoRA. Best in \textbf{bold}.}
  \label{tab:joint_ablation}
  \small
  \begin{tabular}{llcccccc}
    \toprule
    & & GAME-LoRA & Joint-2 & Joint-4 & Joint-Stride & Joint-Fibonacci & Joint-Fib-Offset \\
    \midrule
    & \textit{BT layers} & [19] & [18,19] & [16,17,18,19] & [10,13,16,19,22] & [3,8,13,16,18,19] & [2,7,12,15,17,18,19] \\
    & \textit{LDB layers} & [19] & [18,19] & [16,17,18,19] & [10,13,16,19,22] & [3,8,13,16,18,19] & [3,8,13,16,18,19] \\
    \midrule
    \multicolumn{8}{l}{\textit{Hallucination}} \\
    & HE-Dial & 0.0\% & \textbf{+5.1\%} & -4.3\% & +3.1\% & -0.1\% & +3.6\% \\
    & HE-QA & \textbf{0.0\%} & -0.3\% & -9.5\% & -0.3\% & -1.1\% & -11.8\% \\
    & HE-Summ & 0.0\% & +0.6\% & -17.2\% & -2.3\% & \textbf{+1.0\%} & -3.6\% \\
    & MemoTrap & \textbf{0.0\%} & -2.4\% & -0.4\% & -0.1\% & -1.6\% & -2.4\% \\
    & TFQA-MC1 & \textbf{0.0\%} & -3.5\% & -0.6\% & -0.6\% & -3.5\% & -4.9\% \\
    & TFQA-MC2 & 0.0\% & -1.1\% & -0.2\% & +0.9\% & \textbf{+3.0\%} & -1.1\% \\
    \midrule
    \multicolumn{8}{l}{\textit{Knowledge}} \\
    & NQ & \textbf{0.0\%} & -10.0\% & -10.0\% & -18.5\% & -26.9\% & -30.3\% \\
    & PopQA & \textbf{0.0\%} & -5.4\% & -5.4\% & -4.5\% & -3.5\% & -9.3\% \\
    & WikiText & \textbf{0.0\%} & -0.6\% & -0.6\% & -1.6\% & -2.1\% & -1.3\% \\
    & Winogrande & 0.0\% & +1.4\% & +0.2\% & \textbf{+1.9\%} & -3.3\% & -0.9\% \\
    \midrule
    \multicolumn{8}{l}{\textit{Average}} \\
    & Hallucination & 0.0\% & -0.3\% & -5.4\% & \textbf{+0.1\%} & -0.4\% & -3.4\% \\
    & Knowledge & \textbf{0.0\%} & -3.7\% & -4.0\% & -5.7\% & -8.9\% & -10.4\% \\
    \bottomrule
  \end{tabular}
\end{table*}

\paragraph{Functional separation.}
BT and LDB exhibit distinct layer-sensitivity profiles. BT layer choice strongly affects hallucination
metrics: BT-4 achieves +28.6\% on HE-QA relative to GAME-LoRA, while LDB variations stay within $\pm$8\%.
Conversely, LDB changes more strongly impact knowledge: LDB-Stride degrades NQ by $-$20.2\%, while BT-Stride
only degrades NQ by $-$6.7\%. This asymmetry confirms functional separability: BT primarily addresses
competition externalities (hallucination) while LDB addresses coordination externalities (knowledge retention).

\paragraph{Task-specific depth profiles.}
Contiguous late layers [16--19] (BT-4) maximize factual consistency (+28.6\% HE-QA, +10.2\% MemoTrap),
while sparse cross-depth coverage [10,13,16,19,22] (BT-Stride) maximizes dialogue coherence (+9.7\% HE-Dial).
No single multi-layer configuration dominates across tasks, and all degrade knowledge relative to GAME-LoRA.

\paragraph{Layer alignment.}
Joint BT+LDB at aligned layers outperforms offset configurations. Joint-Fib-Offset (misaligned BT/LDB layers)
produces the worst knowledge degradation ($-$10.4\% avg) and no task improvements despite similar layer counts
to Joint-Fibonacci. Aligned configurations like Joint-Stride achieve marginal hallucination gains (+0.1\% avg).

\paragraph{Pareto optimality.}
Single-layer regularization (GAME-LoRA) achieves the best knowledge preservation across all configurations,
with no multi-layer variant matching its NQ or PopQA scores. Multi-layer BT can boost individual benchmarks
(HE-QA +28.6\%) but degrades knowledge ($-$8.5\% avg) and increases overhead.

%%%%%%%%%%%%%%%%%%%%%%%%%%%%%%%%%%%%%%%%%%%%%%%%%%%%%%%%%%%%%%%%%%%%%%%%%%%%%%%
% LORA LAYER COVERAGE SENSITIVITY
%%%%%%%%%%%%%%%%%%%%%%%%%%%%%%%%%%%%%%%%%%%%%%%%%%%%%%%%%%%%%%%%%%%%%%%%%%%%%%%

\subsection{LoRA Layer Coverage Sensitivity}
\label{app:lora_sensitivity}

\autoref{tab:lora_ablation} compares LoRA coverage strategies relative to GAME-LoRA (all layers 0--23).

\begin{table*}[h]
  \centering
  \caption{LoRA layer coverage ablation: $\Delta$\% vs. GAME-LoRA. Best in \textbf{bold}.}
  \label{tab:lora_ablation}
  \small
  \begin{tabular}{llccc}
    \toprule
    & & GAME-LoRA & Skip 0--5 & Skip 6--10 \\
    \midrule
    & \textit{LoRA layers} & 0--23 & 6--23 & 0--5,11--23 \\
    \midrule
    \multicolumn{5}{l}{\textit{Hallucination}} \\
    & HE-Dial & 0.0\% & -3.6\% & \textbf{+1.2\%} \\
    & HE-QA & \textbf{0.0\%} & -15.0\% & -1.6\% \\
    & HE-Summ & \textbf{0.0\%} & -8.3\% & \textbf{-0.0\%} \\
    & MemoTrap & 0.0\% & -0.4\% & \textbf{+0.4\%} \\
    & TFQA-MC1 & 0.0\% & \textbf{+0.4\%} & -2.0\% \\
    & TFQA-MC2 & 0.0\% & \textbf{+1.0\%} & -1.2\% \\
    \midrule
    \multicolumn{5}{l}{\textit{Knowledge}} \\
    & NQ & 0.0\% & \textbf{+5.1\%} & +0.1\% \\
    & PopQA & 0.0\% & -0.6\% & \textbf{+0.4\%} \\
    & WikiText & \textbf{0.0\%} & -0.3\% & \textbf{-0.0\%} \\
    & Winogrande & \textbf{0.0\%} & -0.2\% & \textbf{+0.0\%} \\
    \midrule
    \multicolumn{5}{l}{\textit{Average}} \\
    & Hallucination & \textbf{0.0\%} & -4.3\% & -0.5\% \\
    & Knowledge & 0.0\% & \textbf{+1.0\%} & +0.1\% \\
    \bottomrule
  \end{tabular}
\end{table*}

Skipping early layers (0--5) improves knowledge (+1.0\% avg, +5.1\% NQ) but degrades hallucination
($-$4.3\% avg). Skipping middle layers (6--10) slightly improves HE-Dial (+1.2\%) and MemoTrap (+0.4\%)
while preserving knowledge (+0.1\% avg). Full coverage (GAME-LoRA) achieves the best hallucination
average, confirming that uniform LoRA application provides the most balanced performance.

%%%%%%%%%%%%%%%%%%%%%%%%%%%%%%%%%%%%%%%%%%%%%%%%%%%%%%%%%%%%%%%%%%%%%%%%%%%%%%%
% ALPHA SCALING SENSITIVITY
%%%%%%%%%%%%%%%%%%%%%%%%%%%%%%%%%%%%%%%%%%%%%%%%%%%%%%%%%%%%%%%%%%%%%%%%%%%%%%%

\subsection{LoRA Alpha Scaling Sensitivity}
\label{app:alpha_sensitivity}

We compare uniform $\alpha=32$ (GAME-LoRA) against a triangular pattern peaking at layer 19 (range 7--57, same average).

\begin{table*}[h]
  \centering
  \caption{LoRA $\alpha$ scaling ablation ($\alpha=32$ avg): $\Delta$\% vs. GAME-LoRA. Best in \textbf{bold}.}
  \label{tab:alpha_ablation}
  \small
  \begin{tabular}{llcc}
    \toprule
    & & GAME-LoRA & Triangular \\
    \midrule
    & \textit{$\alpha$ pattern} & uniform & peak@L19 \\
    \midrule
    \multicolumn{4}{l}{\textit{Hallucination}} \\
    & HE-Dial & 0.0\% & \textbf{+1.4\%} \\
    & HE-QA & \textbf{0.0\%} & -3.2\% \\
    & HE-Summ & 0.0\% & \textbf{+0.2\%} \\
    & MemoTrap & \textbf{0.0\%} & -1.8\% \\
    & TFQA-MC1 & 0.0\% & \textbf{+3.8\%} \\
    & TFQA-MC2 & \textbf{0.0\%} & -0.3\% \\
    \midrule
    \multicolumn{4}{l}{\textit{Knowledge}} \\
    & NQ & \textbf{0.0\%} & -5.0\% \\
    & PopQA & 0.0\% & \textbf{+4.2\%} \\
    & WikiText & \textbf{0.0\%} & -0.2\% \\
    & Winogrande & 0.0\% & \textbf{+0.7\%} \\
    \midrule
    \multicolumn{4}{l}{\textit{Average}} \\
    & Hallucination & \textbf{0.0\%} & \textbf{+0.0\%} \\
    & Knowledge & \textbf{0.0\%} & -0.1\% \\
    \bottomrule
  \end{tabular}
\end{table*}

Triangular scaling improves truthfulness (+3.8\% TFQA-MC1, +4.2\% PopQA) and commonsense (+0.7\% Winogrande)
while uniform scaling preserves factual consistency (HE-QA, MemoTrap) and knowledge (NQ). Both achieve
identical hallucination averages; uniform $\alpha$ provides slightly better knowledge preservation (0\% vs $-$0.1\%).

\end{document}